\def\arxivcopy{}
\documentclass{article}
\usepackage{iclr2027_conference,times}
\usepackage{amsmath,amssymb,amsthm}
\usepackage{graphicx}
\usepackage{booktabs}
\usepackage{longtable}
\usepackage{natbib}
\usepackage{url}
\usepackage[colorlinks=true,allcolors=blue]{hyperref}
\usepackage{tikz}
\usetikzlibrary{positioning,arrows.meta,decorations.pathreplacing}
\definecolor{ovblue}{HTML}{2A78D6}
\definecolor{ovpurple}{HTML}{8A5CD6}
\definecolor{ovgreen}{HTML}{1BAF7A}
\definecolor{ovorange}{HTML}{EB6834}
\definecolor{ovred}{HTML}{C8102E}
\usepackage{xspace}

\newcommand{\AllCertTerminalN}{3{,}748\xspace}

\newcommand{\AllFailPerTime}{5\xspace}

\newcommand{\AllFailTerminal}{14\xspace}

\newcommand{\AllLicPerTime}{74.5\%\xspace}

\newcommand{\AllNConds}{112\xspace}

\newcommand{\AllPerTimeFacts}{14\xspace}

\newcommand{\AllTerminalFacts}{34\xspace}

\newcommand{\AllproNConds}{71\xspace}

\newcommand{\CanonicalLr}{0.002\xspace}
\newcommand{\CompleteRealZThree}{99.0\%\xspace}

\newcommand{\ConfirmFacts}{3\xspace}
\newcommand{\ConfirmN}{43\xspace}

\newcommand{\EarCertRgeFrac}{53.3\%\xspace}

\newcommand{\EarFailTerminal}{7\xspace}
\newcommand{\EarLicAzuma}{22.1\%\xspace}
\newcommand{\EarLicComplete}{31.2\%\xspace}

\newcommand{\EarLicPerfact}{37.1\%\xspace}

\newcommand{\EarLicScreened}{64.2\%\xspace}

\newcommand{\EarLicTerminal}{80.1\%\xspace}
\newcommand{\EarNConds}{41\xspace}
\newcommand{\EarNFacts}{1{,}818\xspace}

\newcommand{\EarTerminalFacts}{14\xspace}
\newcommand{\FixedEventsHi}{210\xspace}
\newcommand{\FixedEventsLo}{0\xspace}

\newcommand{\HistMaxFailZ}{29.1\xspace}
\newcommand{\HistN}{45\xspace}

\newcommand{\HorizonBigCriterion}{met in both\xspace}

\newcommand{\KappaMax}{128\xspace}
\newcommand{\MarginTheta}{0.66\xspace}
\newcommand{\NCanonical}{15/15\xspace}
\newcommand{\NCheckpoints}{6\xspace}
\newcommand{\NFamilies}{3\xspace}
\newcommand{\PerFactTheta}{56\xspace}

\newcommand{\PooledAzumaEvents}{0\xspace}

\newcommand{\PooledCompleteAllowed}{130.8\xspace}
\newcommand{\PooledCompleteEvents}{0\xspace}
\newcommand{\PooledCompleteRuns}{34{,}256\xspace}
\newcommand{\PooledMarginAllowed}{911.6\xspace}
\newcommand{\PooledMarginEvents}{5\xspace}
\newcommand{\PooledMarginRAllowed}{485.6\xspace}
\newcommand{\PooledMarginREvents}{0\xspace}

\newcommand{\PooledPerfactAllowed}{23.5\xspace}
\newcommand{\PooledPerfactEvents}{1\xspace}

\newcommand{\PooledROnlyEvents}{0\xspace}
\newcommand{\PooledROnlyRuns}{44{,}976\xspace}

\newcommand{\PowerFamily}{64\xspace}
\newcommand{\PowerKFive}{8\xspace}
\newcommand{\PowerPEightyFive}{0.304\xspace}
\newcommand{\PowerRuns}{32\xspace}

\newcommand{\ProLicAzuma}{16.4\%\xspace}
\newcommand{\ProLicComplete}{19.9\%\xspace}

\newcommand{\ProLicMargin}{11.5\%\xspace}

\newcommand{\ProLicPerfact}{33.2\%\xspace}

\newcommand{\ProNConds}{47\xspace}

\newcommand{\ProTerminalFacts}{19\xspace}
\newcommand{\RatioBelowForgot}{0\xspace}
\newcommand{\RatioBelowN}{1{,}436\xspace}
\newcommand{\RatioNearForgot}{59\xspace}
\newcommand{\RatioNearN}{61\xspace}
\newcommand{\RatioTopShare}{63.1\%\xspace}
\newcommand{\RhoMax}{0.998\xspace}
\newcommand{\RhoMin}{0.974\xspace}

\newcommand{\RichEvents}{468\xspace}

\newcommand{\RichFailTerminal}{1\xspace}
\newcommand{\RichLicAzuma}{9.6\%\xspace}
\newcommand{\RichLicComplete}{17.0\%\xspace}

\newcommand{\RichLicTerminalStable}{28.1\%\xspace}

\newcommand{\RichLocFacts}{1\xspace}

\newcommand{\RichLocResult}{left undecided\xspace}

\newcommand{\RichN}{24\xspace}

\newcommand{\RichValK}{0\xspace}
\newcommand{\RichValQ}{3\xspace}
\newcommand{\RichValResult}{met\xspace}
\newcommand{\RichZeroEvent}{13\xspace}

\newcommand{\SEOne}{failed\xspace}
\newcommand{\SEThree}{failed\xspace}
\newcommand{\SETwo}{met\xspace}

\newcommand{\SFamilyExit}{44\xspace}
\newcommand{\SFamilyLic}{60.9\%\xspace}
\newcommand{\SFamilyP}{0.28\xspace}
\newcommand{\SFamilyRate}{5.5\%\xspace}
\newcommand{\SFamilyTraj}{800\xspace}

\newcommand{\SNEval}{31\xspace}

\newcommand{\SNTotal}{33\xspace}

\newcommand{\SProbeRuns}{38\xspace}
\newcommand{\SScreenFailZHi}{56\xspace}
\newcommand{\SScreenFailZLo}{32\xspace}

\newcommand{\SShieldFailK}{5\xspace}

\newcommand{\STaskBGain}{25.3\xspace}
\newcommand{\SThresholdQ}{4\xspace}

\newcommand{\SeedQwenEventsHi}{233\xspace}
\newcommand{\SeedQwenEventsLo}{0\xspace}
\newcommand{\SimRuns}{59\xspace}

\newcommand{\SpreadAboveKappa}{59.2\%\xspace}
\newcommand{\SpreadStable}{1.6\xspace}
\newcommand{\SpreadUnstable}{2.7\xspace}
\newcommand{\SpreadViolated}{14.6\xspace}
\newcommand{\SpreadWideN}{421\xspace}
\newcommand{\SpreadWideRge}{94.5\%\xspace}
\newcommand{\StressLrHi}{0.012\xspace}
\newcommand{\StressLrLo}{0.008\xspace}
\newcommand{\StressN}{3\xspace}
\newcommand{\StressRhoHi}{0.995\xspace}
\newcommand{\StressRhoLo}{0.939\xspace}

\newcommand{\TOneK}{0\xspace}
\newcommand{\TOneN}{19\xspace}

\newcommand{\TOneResult}{met\xspace}

\newcommand{\TProbeRuns}{68\xspace}

\newcommand{\TTwoFacts}{2\xspace}

\newcommand{\TTwoResult}{met\xspace}
\newcommand{\TTwoRge}{2\xspace}

\newtheorem{theorem}{Theorem}
\newtheorem{proposition}{Proposition}
\newtheorem{assumption}{Assumption}

\title{When Can First-Order Models of Fine-Tuning\\Bound Forgetting?}
\ifdefined\arxivcopy
\author{Jianchang Su \& Wei Zhang\\University of Connecticut}
\iclrfinalcopy
\else
\author{Anonymous authors\\Paper under double-blind review}
\fi

\begin{document}
\maketitle
\ifdefined\arxivcopy
\lhead{Preprint}
\else
\lhead{Under review as a conference paper at ICLR 2027}
\fi
\begin{abstract}
Fine-tuning a language model on new data can make it forget facts that it
should keep. We ask whether measurements taken at the start of a fine-tuning
run can bound, for each protected fact, the probability that the run makes
the model forget it. In LoRA fine-tuning with stochastic gradient descent on
models from 0.6B to 14B
parameters, a first-order response model estimated by finite-difference
probes predicts changes of per-fact margins with correlation
\RhoMin--\RhoMax. Predictions of forgetting built on this model
nevertheless failed; in Qwen3-0.6B, forgetting requires parameter changes far
outside the region in which the model was validated. The probes can,
however, bound the probability that a margin first falls below a boundary
near zero: we derive Freedman and Azuma first-passage bounds for a linear
surrogate of the margin and test on new runs whether they hold for the
model. The bounds contain a term $R$ that measures how much the response
coefficients change during the run. The simplified Freedman bound, which
sets $R = 0$, certified most facts but was violated in \AllFailTerminal{} of
\AllNConds{} conditions, and in these conditions every fact on which it was
violated had $R \ge a$, where $a$ is the distance of the fact's margin to
the boundary. The complete Freedman bound certifies only facts with $R < a$,
and it held in every condition. On the violated facts, the spread of the
margin across test runs was a median of \SpreadViolated{} times the
prediction of the response model, so the failures are breakdowns of the
model, and in our data they occurred only where $R \ge a$. We found this pattern post hoc and tested it in two preregistered
confirmatory studies with \ConfirmN{} new conditions: the complete bound
held in all of them, and the simplified bound failed there on only
\ConfirmFacts{} facts, each with $R \ge a$. A screen that
decided for a whole condition whether to use the simplified bound failed on
new starting states. First-order models of fine-tuning can thus bound
forgetting on the facts whose response coefficients change by less than
their distance to the boundary.
\end{abstract}

\section{Introduction}
\label{sec:intro}

Fine-tuning a language model on new data can make it forget facts that it
answered correctly before \citep{kirkpatrick2017ewc,luo2023forgetting}, and
sequential knowledge editing has the same effect \citep{gupta2024editing}.
Existing methods predict which examples an update will forget and rank them
by risk \citep{jin2024forecasting,jin2025lowrank}. A ranking, however,
leaves open how likely a given fact is to survive. Before an
update, one would like a guarantee for each protected fact: an upper bound
on the probability that the run makes the model forget it.

This paper asks when such a bound can be computed from measurements taken
at the start of the run. We study LoRA fine-tuning \citep{hu2022lora} with
stochastic gradient descent (SGD) and gradient clipping, and for each protected fact we track its margin,
the log-probability gap between the correct answer and the strongest
alternative. A fact is forgotten if its margin falls below zero within the
horizon. Local models of training, such as influence functions and
linearizations \citep{koh2017influence,pruthi2020tracin,malladi2023kernel},
are accurate for small parameter changes, and our first-order response
model is accurate in the same sense: it predicts margin changes under
held-out perturbations with correlation \RhoMin--\RhoMax{} across
\NCheckpoints{} checkpoints from 0.6B to 14B parameters. Yet five
predictions of forgetting built on this model failed on new data.
Predicting what happens over a run therefore requires more than local
accuracy.

A coupling argument separates the two sources of error and shows what to
bound instead (Section~\ref{sec:theory}). A local model that is run forward
errs by at most its accumulated local error plus the probability that the
process leaves the region in which the model was checked. In Qwen3-0.6B,
forgetting a fact requires a parameter change about 18 times larger than
the largest change at which we validated the model, so the course of
forgetting lies outside what the model describes. The model can, however,
bound the probability
that a margin first reaches a boundary near zero. We derive such
first-passage bounds for a linear surrogate of the margin from Freedman's
inequality \citep{freedman1975tail} and the Azuma inequality, and we test on
new runs whether they hold for the model itself. Both contain a term $R$ that measures how much the
response coefficients change during the run, and all their inputs are
computed before the run (Figure~\ref{fig:overview}). Setting $R = 0$ gives
a simplified Freedman bound that certifies most facts.

Our main finding concerns this simplification. Across \AllNConds{}
conditions, the simplified Freedman bound was violated in
\AllFailTerminal{}, and every one of the \AllTerminalFacts{} facts on which
it was violated had $R \ge a$, where $a$ is the distance of the fact's
margin to the boundary (Figure~\ref{fig:main}a). $R$ comes from the probes
and $a$ from a reference run, so both are known before the run. The
complete Freedman bound, which keeps $R$, certifies only facts with
$R < a$, and both it and the Azuma bound held in every condition. The
failures go beyond the simplification: a bound with the correct
coefficients at each time also failed on facts with $R \ge a$, and on
the violated facts the spread of the margin across test runs was a median
of \SpreadViolated{} times the prediction. The failures are therefore
breakdowns of the first-order model, and they occurred only on facts with
$R \ge a$, although most such facts survived (Section~\ref{sec:evidence}). Because we found this pattern after an interim
look at the data, we preregistered two confirmatory studies on new
conditions (Section~\ref{sec:prospective}). In their \ConfirmN{} conditions,
the complete bound held everywhere, and the simplified bound failed on only
\ConfirmFacts{} facts, each with $R \ge a$, so the location of the failures
rests mainly on the post hoc data.

Before this finding, we had tested a different repair: a screen that uses
the simplified bound only in conditions whose certified facts have a large
median standardized margin. It removed every failure on earlier conditions
but failed on new starting states, where the
complete bound, the Azuma bound and the per-fact rule held. This contrast is
consistent with failures that are confined to individual facts.

\paragraph{Contributions.}
\begin{enumerate}\itemsep2pt
\item Evidence that an accurate first-order response model fails to predict
forgetting over a fine-tuning run, and a coupling argument for what local
probes can bound instead (Sections~\ref{sec:setup} and~\ref{sec:theory}).
\item Freedman and Azuma first-passage bounds on per-fact forgetting with a
coefficient-change term $R$, and the finding that every observed failure of
the first-order model occurred on a fact with $R \ge a$, a condition known
before the run (Sections~\ref{sec:theory} and~\ref{sec:evidence}).
\item Prospective studies on five models from 0.6B to 14B
parameters, CounterFact facts and a doubled horizon, with two confirmatory
tests, a comparison with fitted thresholds and simulation, and a screen that
failed (Sections~\ref{sec:prospective} and~\ref{sec:discussion}).
\end{enumerate}

\section{Related work}
\label{sec:related}

\paragraph{Forgetting during fine-tuning.}
Prior work measures, forecasts, explains or reduces forgetting; we bound its
probability for each fact before the run. Example forgetting was first
measured during standard training \citep{toneva2019forgetting}. In language
models, forgetting has been measured empirically \citep{luo2023forgetting},
related to fine-tuning loss by scaling laws
\citep{kalajdzievski2024scaling}, and found smaller under LoRA than under
full fine-tuning \citep{biderman2024lora}. New knowledge also affects what
the model already knows \citep{gekhman2024newknowledge}, and sequential
editing leads to gradual and then abrupt forgetting
\citep{gupta2024editing}. Forecasting methods rank the upstream examples
that an update will forget, from logit changes \citep{jin2024forecasting}
or low-rank task-example associations \citep{jin2025lowrank}, rather than
bound a probability. Theory explains forgetting through the overlap of
neural tangent kernels \citep{doan2021ntk} and in linear regression
\citep{evron2022linear}. Mitigation changes the update: elastic weight
consolidation \citep{kirkpatrick2017ewc}, replay of maximally interfered
examples \citep{aljundi2019mir}, and null-space \citep{fang2025alphaedit} or
low-curvature \citep{ikram2026crispedit} projection. Our bounds leave the
update unchanged.

\paragraph{Local models of training.}
Influence functions \citep{koh2017influence}, gradient tracing
\citep{pruthi2020tracin} and kernel linearizations of fine-tuning
\citep{malladi2023kernel} describe training locally, and interpretability
methods built on them are usually validated on short horizons
\citep{bereska2024review,mueller2025mib}. We ask what such a local model
supports over a full fine-tuning run.

\paragraph{Bounds and risk control.}
Our bounds apply maximal martingale inequalities
\citep{freedman1975tail} at a fixed horizon; time-uniform versions
\citep{howard2020timeuniform} would allow monitoring during the run.
Generalization bounds for language models bound a population risk after
training \citep{lotfi2024nonvacuous}, whereas we bound a per-fact event
during a run. Conformal prediction
\citep{vovk2005algorithmic,angelopoulos2021gentle}, risk-controlling
prediction sets \citep{bates2021rcps}, conformal risk control
\citep{angelopoulos2024crc}, Learn-then-Test \citep{angelopoulos2025ltt} and
selective prediction \citep{geifman2017selective} calibrate thresholds on
exchangeable data. Our conditions were chosen by study design, which rules
out exchangeability, so we fixed the thresholds of our rules on earlier
conditions and tested them on new ones.

\section{Setting and the first-order response model}
\label{sec:setup}

\paragraph{Two-task fine-tuning.}
A model is first fine-tuned on task A, which contains the protected facts,
and then on task B. We use two data sources. With \emph{synthetic facts},
task A teaches 320 associations between invented entities and one of four
answer labels, each in two paraphrase templates. We protect 64 of them and
query each with a third template reserved for evaluation. Task B
teaches 128 new associations of the same form. With \emph{CounterFact}
\citep{meng2022rome}, the 64 protected facts are true facts that the
pretrained model already answers correctly, so task A is skipped, and task B
trains on 256 counterfactual rewrites of other subjects, which is knowledge
editing by fine-tuning. Llama-3.2-1B and Gemma-3-12B denote the
instruction-tuned checkpoints. All models are adapted with LoRA of rank 8 on all
linear layers \citep{hu2022lora} over 4-bit weights
\citep{dettmers2023qlora}. Task A uses AdamW. A run proceeds only if the
accuracy on the protected facts at the start of task B reaches a
fixed threshold (0.75 for synthetic facts, 0.95 for CounterFact).
Task B uses SGD with batch size 8 (4 for CounterFact), learning rate $\eta$,
and gradient clipping at norm $G_0 = 1$. We call the parameters at the
start of task B the \emph{starting state}. A \emph{condition} is a model, a
data source, a task-B learning rate and a seed; a \emph{configuration} is
the model, data source and learning rate of a condition. The seed determines task-A training and the
minibatch orders of the test runs defined below.

\paragraph{Margin and forgetting.}
For protected fact $i$ with correct answer $y_i$ among four candidates, the
margin after $t$ task-B steps is
$Y_i(t) = \log p_{\theta_t}(y_i \mid x_i) - \max_{y \neq y_i} \log
p_{\theta_t}(y \mid x_i)$. It is positive exactly when the model answers
correctly. We evaluate margins on the grid $G = \{0, 10, \dots, T\}$ with
horizon $T = 40$ steps ($T = 80$ in one study), and a \emph{forgetting
event} occurs if $Y_i(t) < 0$ for some $t \in G$. The \emph{reference run}
is one task-B run with a fixed minibatch order, and $m_i(t)$ is the margin
along it. A fact is \emph{eligible} if $\min_{t \in G} m_i(t) > 0$. Other
facts are already forgotten on the reference run, and our claims concern
eligible facts only. A \emph{test run} is a further task-B run from the same
starting state with its own minibatch order. We use 32 test runs per
condition (24 in two earlier conditions) to measure how often each fact is
forgotten.

\paragraph{First-order response model.}
Let $b_1, \dots, b_k$ ($k = 8$) be an orthonormal basis of the leading
directions of the margin gradients of 16 protected facts at the starting
state; the basis is computed from the starting state alone. For a perturbation time
$s \in G$ and a later time $t$, the response coefficients
$q_{i,t,s} \in \mathbb{R}^k$ are central finite differences,
$q_{i,t,s,j} = \bigl(Y_i^{+j}(t) - Y_i^{-j}(t)\bigr)/2\epsilon$, where
$Y_i^{\pm j}$ is the margin when $\pm\epsilon b_j$ ($\epsilon = 0.02$) is
added to the parameters at step $s$ of the reference run and the same
minibatch order is replayed up to step $t$. The first-order response model
predicts that a perturbation $\sum_j u_j b_j$ added at step $s$ changes the
margin at time $t$ by $q_{i,t,s}^\top u$. We call the finite-difference
measurements the \emph{probes}. The probes also estimate, at each
perturbation time $s$, the covariance $\Sigma_s$ of the clipped task-B
gradients projected on the basis. Each probe replays the run from a perturbation time to $T$, so
the probes cost $O(k\,|G|\,T)$ SGD steps, which grows with the square of the
horizon. All quantities used by the bounds are computed before the first
test run (Figure~\ref{fig:overview}).

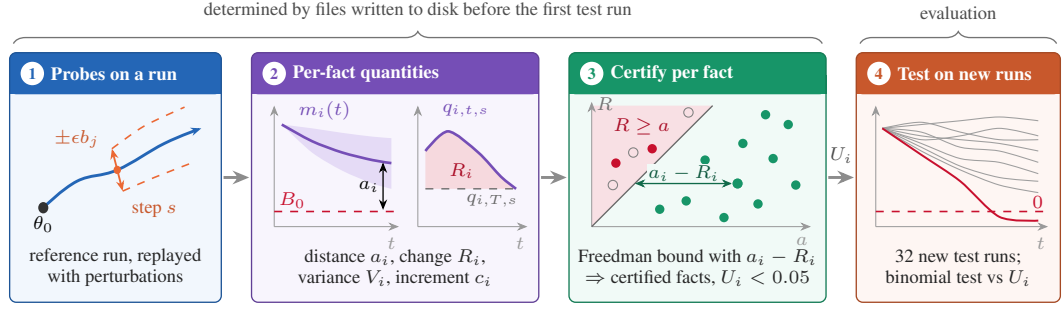
\begin{figure}[t]
\centering
\begin{tikzpicture}[x=1mm, y=1mm, font=\scriptsize,
  >={Stealth[length=1.4mm,width=1.1mm]},
  stagenote/.style={align=center, inner sep=0pt, text=black!85},
  axis/.style={black!40, line width=0.35pt},
  flow/.style={-{Stealth[length=2.2mm,width=1.8mm]}, line width=1.1pt, black!45},
  bracket/.style={decorate, decoration={brace, amplitude=1.6mm}, black!55,
    line width=0.5pt},
]
\def\H{33}
\def\panel#1#2#3#4#5{%
  \fill[#3!6, rounded corners=2.5pt] (#1,0) rectangle ++(#2,\H);
  \fill[#3!85!black, rounded corners=2.5pt] (#1,\H-5.5) rectangle ++(#2,5.5);
  \fill[#3!85!black] (#1,\H-5.5) rectangle ++(#2,2.5);
  \draw[#3!70!black, line width=0.5pt, rounded corners=2.5pt]
    (#1,0) rectangle ++(#2,\H);
  \node[circle, fill=white, inner sep=0.45pt, minimum size=3.1mm,
    font=\tiny\bfseries, text=#3!85!black] at (#1+3,\H-2.75) {#4};
  \node[anchor=west, font=\scriptsize\bfseries, text=white, inner sep=0pt]
    at (#1+5.4,\H-2.75) {#5};
}

\panel{0}{28}{ovblue}{1}{Probes on a run}
\begin{scope}[xscale=0.89]
\draw[line width=1.1pt, ovblue!85!black, ->]
  plot[smooth, tension=0.7] coordinates {(5,12.5) (10,16) (16,17.5) (23,21) (28.5,23)};
\fill[black!80] (5,12.5) circle (0.75);
\node[anchor=north, inner sep=0.8pt] at (5,11.7) {$\theta_0$};
\draw[ovorange, dashed, line width=0.6pt]
  plot[smooth] coordinates {(15.2,20.3) (20,23.3) (26,25.8)};
\draw[ovorange, dashed, line width=0.6pt]
  plot[smooth] coordinates {(16.8,14.7) (21.5,16.5) (27,18.4)};
\draw[ovorange, line width=0.7pt, ->] (16,17.5) -- (15.2,20.3);
\draw[ovorange, line width=0.7pt, ->] (16,17.5) -- (16.8,14.7);
\fill[ovorange] (16,17.5) circle (0.55);
\node[text=ovorange!85!black] at (10,21.6) {$\pm\epsilon b_j$};
\node[text=ovorange!85!black] at (21.5,12.3) {step $s$};
\end{scope}
\node[stagenote] at (14,4.6) {reference run, replayed\\
  with perturbations};

\begin{scope}[xshift=32mm]
\panel{0}{38}{ovpurple}{2}{Per-fact quantities}
\draw[axis, ->] (3,10) -- (19.5,10) node[below left=0.2mm and -0.8mm] {$t$};
\draw[axis, ->] (3,10) -- (3,26);
\fill[ovpurple!22]
  plot[smooth] coordinates {(4,23.5) (8,23) (12,22.4) (16,21.9) (18.5,21.7)}
  -- plot[smooth] coordinates {(18.5,15) (16,15.7) (12,17.4) (8,19.8) (4,23.5)}
  -- cycle;
\draw[ovpurple!85!black, line width=1pt]
  plot[smooth] coordinates {(4,23.5) (8,21.3) (12,19.8) (16,18.8) (18.5,18.4)};
\draw[ovred, dashed, line width=0.6pt] (3,12) -- (19.5,12);
\node[text=ovred, anchor=south west, inner sep=0.5pt] at (3.4,12.2) {$B_0$};
\draw[<->, line width=0.5pt] (17.4,12.2) -- (17.4,18.5);
\node[anchor=east, inner sep=0.8pt] at (17.2,15.3) {$a_i$};
\node[text=ovpurple!85!black, anchor=west] at (5,25.4) {$m_i(t)$};
\draw[axis, ->] (22,10) -- (36.5,10) node[below left=0.2mm and -0.8mm] {$t$};
\draw[axis, ->] (22,10) -- (22,26);
\fill[ovred!16]
  plot[smooth] coordinates {(23,20) (26,22.6) (30,19.8) (33,16.6) (35,15)}
  -- (35,15) -- (23,15) -- cycle;
\draw[black!60, dashed, line width=0.5pt] (23,15) -- (35.5,15);
\draw[ovpurple!85!black, line width=1pt]
  plot[smooth] coordinates {(23,20) (26,22.6) (30,19.8) (33,16.6) (35,15)};
\node[text=ovred!85!black] at (27.8,17.4) {$R_i$};
\node[text=ovpurple!85!black, anchor=west] at (24,25.4) {$q_{i,t,s}$};
\node[text=black!60, anchor=north east, inner sep=0.5pt] at (35.5,14.6)
  {$q_{i,T,s}$};
\node[stagenote] at (19,4.6) {distance $a_i$, change $R_i$,\\
  variance $V_i$, increment $c_i$};
\end{scope}

\begin{scope}[xshift=74mm]
\panel{0}{34}{ovgreen}{3}{Certify per fact}
\fill[ovred!13] (3,10) -- (19,26) -- (3,26) -- cycle;
\draw[axis, ->] (3,10) -- (32,10) node[below left=0.2mm and -0.8mm] {$a$};
\draw[axis, ->] (3,10) -- (3,26.5);
\node[text=black!55, anchor=west, inner sep=0.5pt] at (3.4,26.2) {$R$};
\draw[black!60, line width=0.5pt] (3,10) -- (19,26);
\node[text=ovred!85!black] at (9.6,23.3) {$R \ge a$};
\foreach \p in {(12,12.3), (15.5,13.4), (20.5,11.2), (25,13), (27,19),
                (23,21), (29,23.5), (18,18.8), (30,15.5), (26,24.8)}
  \fill[ovgreen!85!black] \p circle (0.62);
\foreach \p in {(5.8,15.5), (8.5,19.8), (15.8,25.1)}
  \draw[black!50, line width=0.4pt] \p circle (0.62);
\foreach \p in {(6.2,18.4), (11,20.3)}
  \fill[ovred] \p circle (0.62);
\draw[<->, line width=0.5pt, ovgreen!60!black] (8.7,15.7) -- (21.6,15.7);
\fill[ovgreen!85!black] (22.3,15.7) circle (0.75);
\draw[white, line width=0.3pt] (22.3,15.7) circle (0.75);
\node[text=ovgreen!50!black, fill=ovgreen!6, inner sep=0.4pt] at (15.2,17.1)
  {$a_i - R_i$};
\node[stagenote] at (17,4.6) {Freedman bound with $a_i - R_i$\\
  $\Rightarrow$ certified facts, $U_i < 0.05$};
\end{scope}

\begin{scope}[xshift=112mm]
\panel{0}{27}{ovorange}{4}{Test on new runs}
\begin{scope}[xscale=0.86]
\draw[axis, ->] (3,10) -- (29.5,10) node[below left=0.2mm and -0.8mm] {$t$};
\draw[axis, ->] (3,10) -- (3,26);
\foreach \a/\b/\c/\d in {22/21.5/21/20, 21/19/18.5/17, 23.5/24/24.5/24,
    20.5/18/16.5/16, 22.5/20.5/20/18.5, 21.5/22/23/22, 20/17.5/15.5/14.5,
    23.8/22.5/21.5/21.5}
  \draw[black!38, line width=0.4pt]
    plot[smooth] coordinates {(4,23) (10,\a) (16,\b) (22,\c) (28,\d)};
\draw[ovred, line width=0.75pt]
  plot[smooth] coordinates {(4,23) (10,19.5) (16,16) (22,11.3) (28,10.8)};
\draw[ovred, dashed, line width=0.6pt] (3,12) -- (29,12);
\node[text=ovred, anchor=south east, inner sep=0.5pt] at (29,12.2) {$0$};
\end{scope}
\node[stagenote] at (13.5,4.6) {32 new test runs;\\
  binomial test vs $U_i$};
\end{scope}

\draw[flow] (28.3,\H/2) -- (31.7,\H/2);
\draw[flow] (70.3,\H/2) -- (73.7,\H/2);
\draw[flow] (108.3,\H/2) -- (111.7,\H/2)
  node[midway, above=0.6mm, text=black!70] {$U_i$};
\draw[bracket] (0.5,\H+1.2) -- (107.5,\H+1.2)
  node[midway, above=2mm, text=black!70]
  {determined by files written to disk before the first test run};
\draw[bracket] (112.5,\H+1.2) -- (138.5,\H+1.2)
  node[midway, above=2mm, text=black!70] {evaluation};
\end{tikzpicture}
\caption{The procedure for one condition (schematic). (1) A reference run
gives the margins $m_i(t)$, and probes along it give the response
coefficients $q_{i,t,s}$ and the gradient covariance $\Sigma_s$. (2) These
determine, for each fact, the distance $a_i$ to the boundary, the variance
estimate $V_i$, the increment bound $c_i$ and the coefficient change $R_i$,
the gap between the coefficients at time $t$ and at the end $T$. (3) The
complete bound uses $a_i - R_i$, so it certifies only facts with
$R_i < a_i$ (Sections~\ref{sec:theory} and~\ref{sec:method}). (4) A
binomial test compares the forgetting events in the test runs with $U_i$.}
\label{fig:overview}
\end{figure}

\paragraph{Accuracy of the response model.}
We test the model on held-out perturbations, random combinations of the
basis directions whose effect is measured separately and compared with the
prediction $q_{i,T,0}^\top u$. At task-B learning rate \CanonicalLr,
Table~\ref{tab:local} (Appendix~\ref{app:local}) reports \NCanonical{} passing validation runs across
\NCheckpoints{} checkpoints and \NFamilies{} families (Qwen3, Llama 3.2,
Gemma 3), with $\rho$ between \RhoMin{} and \RhoMax. At the higher learning
rates \StressLrLo{} to \StressLrHi, \StressN{} further runs gave $\rho$
between \StressRhoLo{} and \StressRhoHi. Two requirements must hold. The
basis must come from gradients: a random low-rank basis gives
$\rho = 0.057$, because the margin changes that it induces are smaller than
the numerical noise of the forward pass. Task B must use SGD, under which
the probe responses scale linearly with the perturbation size over the full
horizon, whereas under Adam they stop doing so after about ten steps
(Appendix~\ref{app:local}). Despite this accuracy, five preregistered
predictions of forgetting built on the model failed on new data
(Appendix~\ref{app:gap}), and Section~\ref{sec:theory} explains why.

\section{Bounds from local probes}
\label{sec:theory}

\subsection{What a local model can determine about a run}

We first ask what a local model can determine about a whole run. Let $P$ be
the law of the training process and $Q$ the law obtained by running a local
model forward. Suppose that the two transition kernels agree to within
$\varepsilon_t$ in total variation on a region $\mathcal N$ that contains
the starting state, and let $\tau$ be the time at which the process first
leaves $\mathcal N$.

\begin{theorem}[Coupling bound]
\label{thm:coupling}
For any event $A$ determined by the path up to $T$,
$|P(A) - Q(A)| \le P(\tau \le T) + \sum_{t \le T} \varepsilon_t$, and both
terms are attained.
\end{theorem}

\begin{theorem}[Non-identifiability]
\label{thm:nonid}
For any neighborhood $\mathcal N$, order $d$, and $p \in (0,1]$, there exist
smooth systems that agree on $\mathcal N$ in value, in all derivatives up to
order $d$, and in noise law, but whose probabilities of an event determined
at step 2 differ by $p$.
\end{theorem}

Theorem~\ref{thm:coupling} separates the accumulated local error from the
probability of leaving the region in which the local model was checked; we
use it to structure the problem, because response correlations measure
accuracy along a few directions rather than a total-variation distance. In
Qwen3-0.6B, a single forgetting event needs a parameter change about
$18\times$ larger than the largest change at which the response model was
validated (Appendix~\ref{app:local}), and by Theorem~\ref{thm:nonid} what
happens outside $\mathcal N$ is unidentifiable from local data. We
therefore bound the probability that a linear surrogate of the margin first
crosses a boundary, a concentration problem, and test on new runs whether
the bound transfers to the model (proofs in Appendix~\ref{app:theory}).

\subsection{Two first-passage bounds}

Fix a protected fact and drop the index $i$. The linear surrogate applies the
first-order response model to the noise of the task-B updates:
$\widehat Y_t = m_t + \sum_{s \le t} \eta\, q_{t,s}^\top \xi_s$, where
$\xi_s$ is the clipped gradient at step $s$, projected on the basis and
centered at its conditional mean. A clipped gradient and its conditional
mean each have norm at most $G_0$, so $\|\xi_s\| \le 2G_0$. Let
$S_t = \sum_{s \le t} \eta\, q_{T,s}^\top \xi_s$ be the noise term with the
coefficients of the final time $T$. It is a martingale whose increments are
bounded by
\begin{equation}
c = 2\eta G_0 r, \qquad r = \max_s \|q_{T,s}\|.
\label{eq:c}
\end{equation}
Let $a = \min_{t \in G} m_t - B_0$ be the distance of the reference margin
from the boundary $B_0 = 0.01$. Let $K$ be the predictable quadratic
variation of $S_T$ estimated from the probes, and let $V = \kappa K$ with
$\kappa = 2$. Let $R = 2\eta G_0 \max_{t \in G}\sum_{s<t} w_s \|q_{t,s} -
q_{T,s}\|$ bound the change of the response coefficients over the grid,
where $w_s$ is the width of the grid interval that starts at $s$.

\begin{theorem}[First-passage bounds on the grid]
\label{thm:tu}
Let $a' = a - R > 0$. If the increments of $S$ are bounded by $c$, then
\[
\Pr\bigl(\exists t \in G: \widehat Y_t < B_0\bigr) \le
U^{\mathrm A} = \exp\!\Bigl(-\frac{a'^2}{2Tc^2}\Bigr).
\]
If, in addition, the predictable quadratic variation of $S$ is at most $V$,
the same probability is at most
$U^{\mathrm F} = \exp\!\bigl(-a'^2/(2(V + c\,a'/3))\bigr)$.
\end{theorem}

The proof writes $\widehat Y_t - m_t = S_t + D_t$ with $|D_t| \le R$ and
applies the maximal forms of the Azuma and Freedman inequalities, which
cover all grid times at once. We call $U^{\mathrm F}$ the Freedman
bound and $U^{\mathrm A}$ the Azuma bound. A bound \emph{certifies} an
eligible fact if its value is below 0.05; this is a statement about one
fact, and a statement about a set $\mathcal L$ of facts at level $\alpha$
requires $\sum_{i\in\mathcal L} U_i \le \alpha$. A forgetting event implies a
crossing of $B_0$, so each bound also bounds the probability of forgetting
under the surrogate; for the actual model, a certified fact is a prediction
that we test on new runs.

\paragraph{The evaluated bounds.}
The \emph{complete} Freedman bound uses the measured $R$. The
\emph{simplified} Freedman bound and the Azuma bound, as we evaluate it, set
$R = 0$, so Theorem~\ref{thm:tu} covers them for the surrogate $m_t + S_t$
with final-time coefficients, which differs from $\widehat Y_t$ by up to
$R$; they are justified only for facts whose coefficient change is small
compared with $a$. The Azuma bound replaces the variance estimate by the
worst case $Tc^2$ and is much more conservative.

\subsection{Assumptions of each bound}

\begin{assumption}[Uniform remainder]\label{as:local}
With probability at least $1-\delta_r$,
$\sup_{t\in G}|Y_t-\widehat Y_t| \le \gamma$, where $Y_t$ is the actual
margin. Under this assumption, each bound applies to the actual margin with
$a' - \gamma$ in place of $a'$, plus $\delta_r$.
\end{assumption}
\begin{assumption}[Bounded increments]\label{as:inc}
The increments of $S$ are bounded by $c$ in Equation~\ref{eq:c}. Both bounds
need this assumption, and gradient clipping enforces it.
\end{assumption}
\begin{assumption}[Variance estimate]\label{as:var}
$V$ is at least the realized predictable quadratic variation of $S$. Only
the Freedman bounds need this assumption.
\end{assumption}

Assumption~\ref{as:local} remains unverified, and a high response
correlation is insufficient to establish it. Whether the bounds hold for the
actual model is therefore an empirical question, which we test on new runs;
Section~\ref{sec:evidence} shows on which facts the assumptions fail.
Estimation error in $V$ can be absorbed at a stated confidence level
(Appendix~\ref{app:theory}).

\section{Rules that certify facts}
\label{sec:method}

A rule computes a value $U_i$ for each fact before the run and certifies
the eligible facts with $U_i < 0.05$; each bound of
Section~\ref{sec:theory} is such a rule. Before we understood where the
simplified bound fails, we designed on earlier conditions two rules that
keep part of its certified facts, and one baseline that uses only the
reference run. The screen threshold was fixed before the prospective
studies, and the other two thresholds after the first prospective condition
had finished (Section~\ref{sec:prospective}).

\paragraph{Screened bound.}
The first rule decides for a whole condition. Let $\mathcal C$ be the set of
eligible facts that the simplified bound certifies, let
$z_i = a_i/\sqrt{V_i}$ be the standardized margin of fact $i$, and let
$\tilde z$ be the median of $z_i$ over $\mathcal C$. The \emph{screen} is
the test $\tilde z \ge 30$. If the condition passes the screen, the
\emph{screened bound}\footnote{Our code calls this rule ValidityShield.} uses the simplified Freedman bound for every fact of the
condition (the \emph{Freedman branch}); otherwise it uses the Azuma bound
(the \emph{Azuma branch}), on the rationale that large standardized margins
tolerate an underestimated variance. Because the screen uses only
quantities that exist before the run, the selected bound holds whenever
each branch bound holds (Proposition~\ref{prop:select}).

\paragraph{Per-fact rule and margin threshold.}
The \emph{per-fact rule} decides for each fact. It certifies a fact if the
simplified bound certifies it and its own $z_i$ is at least
$\theta_z = \PerFactTheta$. The \emph{margin threshold} is a baseline that
uses only the reference run. It certifies a fact if its distance $a_i$ is at
least $\theta = \MarginTheta$, and it claims a probability of 0.05 for that
fact. Both thresholds are the smallest values at which every earlier
condition passes the test below, and the screen threshold comes from an
earlier screening rule (Appendix~\ref{app:screen}).

\paragraph{Testing a rule.}
For each certified fact, a one-sided binomial test compares the number of
test runs in which the fact is forgotten with its value $U_i$ (with 0.05 for
the margin threshold), with a Bonferroni correction over the 64 facts of the
condition. A rule is \emph{violated on a fact} if this test is significant,
and a condition is \emph{violating} if the rule is violated on at least one
of its facts. A condition is \emph{evaluable} if it has at least 8 eligible
facts.

\section{Where the simplified bound fails}
\label{sec:evidence}

We first examine the \EarNConds{} of the \HistN{} earlier conditions that
have stored probe files, from which $R$ can be computed
(Appendix~\ref{app:history}); they contain \EarNFacts{} eligible facts.
Their outcomes existed before we defined the rules of
Section~\ref{sec:method}, so this analysis is post hoc (first row of
Table~\ref{tab:main}).

\begin{figure}[t]
\centering
\includegraphics[width=\textwidth]{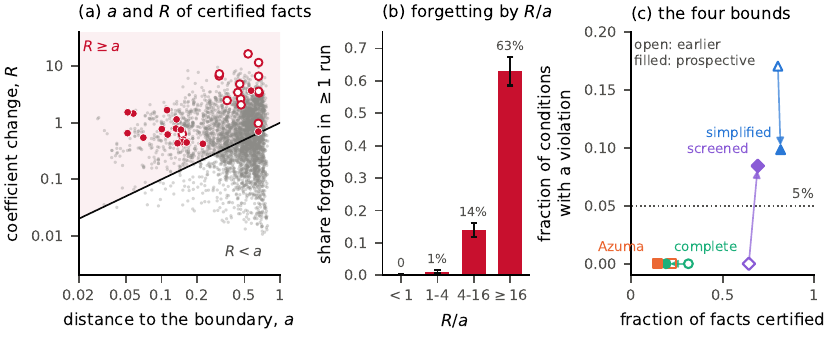}
\caption{\textbf{(a)} The \AllCertTerminalN{} facts that the simplified
Freedman bound certifies, by their distance $a$ to the boundary and
coefficient change $R$, both computed before the run. Red: facts on which
the simplified bound was violated, in earlier (open) and prospective
(filled) conditions; all lie in the shaded region $R \ge a$, which the
complete bound excludes. \textbf{(b)} Share of eligible facts forgotten in
at least one test run, by $R/a$, with Wilson 95\% intervals
(Appendix~\ref{app:ratio}). \textbf{(c)} For the four bounds, the fraction
of eligible facts certified against the fraction of conditions with a
violation, on the \EarNConds{} earlier (open, post hoc) and
\AllproNConds{} prospective (filled) conditions.}
\label{fig:main}
\end{figure}

The simplified Freedman bound certified \EarLicTerminal{} of eligible
facts, but it was violated in \EarFailTerminal{} of the \EarNConds{}
conditions. Figure~\ref{fig:main}a shows where the violations occurred. Each
of the \EarTerminalFacts{} facts on which it was violated in these
conditions lies above the line $R = a$: its response coefficients changed
during the run by more than the distance of its margin to the boundary. The
complete bound replaces $a$ by $a - R$, so it certifies only facts with
$R < a$, and it held in every condition.

In Theorem~\ref{thm:tu}, $R$ bounds the difference between the surrogates
with final-time and with per-time coefficients, but the failures have a
different source. After these outcomes were known, we evaluated a
\emph{per-time bound}, which applies Theorem~\ref{thm:tu} at each grid time
with the coefficients of that time and combines the grid times with a union
bound, so that the term $R$ drops out. Over all \AllNConds{} conditions,
including the prospective ones of Section~\ref{sec:prospective}, it certified
\AllLicPerTime{} of eligible facts but was violated in \AllFailPerTime{}
conditions, and each of the \AllPerTimeFacts{} facts on which it was
violated had $R \ge a$. On such
facts, the first-order model fitted along the reference run describes the
test runs poorly: on the facts that violated the simplified bound, the
standard deviation of the final margin across test runs was a median of
\SpreadViolated{} times the prediction, against \SpreadStable{} on facts
with $R < a$ (Appendix~\ref{app:spread}). A large $R$ means that the
response of the margin to a perturbation changes strongly during the run,
so a model fitted along one run is least likely to describe other runs. We
therefore read $R \ge a$, computed before the run, as a necessary condition
for such a failure in our data. It is far from sufficient, because the
median ratio over all facts with $R \ge a$ is \SpreadUnstable, and
forgetting rises steadily with $R/a$ above 1 (Appendix~\ref{app:ratio}).
The complete bound holds because it excludes every fact that meets the
condition.

The complete bound therefore certifies fewer facts. In our data, $R \ge a$
was necessary for a violation but far from sufficient: \EarCertRgeFrac{} of
the facts that the simplified bound certified had $R \ge a$, and most of
them survived every test run. The complete bound excludes all of them and
certified \EarLicComplete{} of eligible facts, more than the Azuma bound
(\EarLicAzuma).

The screened bound and the per-fact rule appeared to avoid this loss. Every
earlier condition with a violation of the simplified bound had a median
standardized margin below 30 (at most \HistMaxFailZ;
Appendix~\ref{app:screen}), so the screened bound held there and certified
\EarLicScreened{} of eligible facts; the per-fact rule (\EarLicPerfact) and
the margin threshold held by construction. Only new starting states can
test these rules.

\section{Prospective evaluation}
\label{sec:prospective}

\subsection{Design}

We ran four prospective studies on new starting states, each fixed with
its analysis script before its first run. The \emph{main study} has
\SNTotal{} conditions in six condition sets: Qwen3-0.6B, Llama-3.2-1B,
Qwen3-1.7B and Qwen3-8B on synthetic facts, and Qwen3-0.6B and Llama-3.2-1B
on CounterFact facts that the pretrained model already answers correctly,
with new seeds (71, 73 and 79) and learning rates that include rates at
which the simplified bound had failed before. The \emph{horizon study}
reruns eight of its starting states for $T = 80$ steps, and the \emph{14B
study} runs Qwen3-14B with new seeds. Four Qwen3-8B and Qwen3-14B
conditions fell below the task-A accuracy threshold and were stopped before
any bound was computed; a further batch with two new seeds was run for each
model, and the Qwen3-8B row of Table~\ref{tab:main} consists of this batch.
Before the first test run of each condition, our code wrote the probe
results, the simplified, Azuma and screened bounds and the branch to disk;
the other rules are deterministic functions of these files.

The main question was the screened bound, judged by three criteria. The
\emph{violation criterion} allows at most the 95th percentile of
$\mathrm{Binomial}(N, 0.05)$ violating conditions among $N$ evaluable ones,
a count that a rule violating each condition with probability 0.05 exceeds
with probability at most 0.05. The \emph{certification criterion} requires
that the screened bound certify at least 40\% of eligible facts and twice
as many as the Azuma bound, and the \emph{assignment criterion} that every
condition in which the simplified bound fails be assigned to the Azuma
branch. The per-fact rule and the margin threshold were added after the
first condition had finished.

After an interim look at 15 conditions, we added a confirmatory test on the
conditions whose test runs started later (28 runs excluded): the complete
bound must meet the violation criterion, and the \emph{location criterion}
requires that at least 90\% of the facts on which the simplified bound is
violated have $R \ge a$. Because few such facts occurred, a \emph{second
confirmatory study} applied the same criteria, with at least five violated
facts needed for a decision, to 24 new conditions: Qwen3-0.6B at learning
rates 0.012, 0.018 and 0.021 and Qwen3-1.7B at 0.012 on synthetic facts,
with six new seeds.

\begin{table}[t]
\centering\small
\caption{Main results. For each condition set and rule: the number of
violating conditions (viol.) and the percentage of eligible facts that the
rule certifies (\%). $n$ is the number of evaluable conditions, and events
is the number of forgetting events of eligible facts over all test runs. The
rows of condition sets contain the main, horizon and 14B studies; the
Qwen3-8B row contains the added batch (seeds 79 and 83). The confirmatory
set contains the conditions of these studies whose test runs started after
the confirmatory test was defined.}
\label{tab:main}
\vspace{2pt}
\setlength{\tabcolsep}{2.2pt}
\begin{tabular}{lrrrrrrrrrrrrrr}
\toprule
& & & \multicolumn{2}{c}{Simplified} & \multicolumn{2}{c}{Screened} & \multicolumn{2}{c}{Complete} & \multicolumn{2}{c}{Azuma} & \multicolumn{2}{c}{Per-fact} & \multicolumn{2}{c}{Margin} \\
& & & \multicolumn{2}{c}{Freedman} & \multicolumn{2}{c}{bound} & \multicolumn{2}{c}{Freedman} & \multicolumn{2}{c}{bound} & \multicolumn{2}{c}{rule} & \multicolumn{2}{c}{threshold} \\
\cmidrule(lr){4-5}\cmidrule(lr){6-7}\cmidrule(lr){8-9}\cmidrule(lr){10-11}\cmidrule(lr){12-13}\cmidrule(lr){14-15}
Condition set & $n$ & events & viol. & \% & viol. & \% & viol. & \% & viol. & \% & viol. & \% & viol. & \% \\
\midrule
Earlier conditions (post hoc) & 41 & 1{,}795 & 7 & 80 & 0 & 64 & 0 & 31 & 0 & 22 & 0 & 37 & 0 & 14 \\
\midrule
Qwen3-0.6B{,} synthetic & 9 & 603 & 2 & 85 & 2 & 85 & 0 & 29 & 0 & 8 & 0 & 53 & 0 & 6 \\
Llama-3.2-1B{,} synthetic & 6 & 319 & 1 & 71 & 0 & 47 & 0 & 18 & 0 & 18 & 0 & 21 & 0 & 29 \\
Qwen3-1.7B{,} synthetic & 4 & 337 & 1 & 87 & 1 & 49 & 0 & 6 & 0 & 12 & 0 & 11 & 0 & 6 \\
Qwen3-8B{,} synthetic & 2 & 8 & 0 & 94 & 0 & 94 & 0 & 37 & 0 & 41 & 0 & 38 & 0 & 9 \\
Qwen3-0.6B{,} CounterFact & 6 & 279 & 0 & 74 & 0 & 58 & 0 & 3 & 0 & 10 & 0 & 27 & 0 & 0 \\
Llama-3.2-1B{,} CounterFact & 6 & 215 & 2 & 80 & 2 & 80 & 0 & 10 & 0 & 13 & 0 & 28 & 0 & 0 \\
Qwen{,} Llama{,} $T = 80$ & 8 & 55 & 0 & 94 & 0 & 78 & 0 & 18 & 0 & 10 & 0 & 44 & 0 & 21 \\
Qwen3-14B{,} synthetic & 6 & 144 & 0 & 89 & 0 & 66 & 0 & 39 & 0 & 32 & 0 & 33 & 0 & 21 \\
\midrule
Confirmatory set & 19 & 785 & 2 & 80 & 2 & 62 & 0 & 15 & 0 & 12 & 0 & 27 & 0 & 3 \\
Second confirmatory study & 24 & 468 & 1 & 77 & 1 & 69 & 0 & 17 & 0 & 10 & 1 & 40 & 0 & 13 \\
All prospective & 71 & 2{,}428 & 7 & 82 & 6 & 69 & 0 & 19 & 0 & 14 & 1 & 35 & 0 & 12 \\
\bottomrule
\end{tabular}

\end{table}

\subsection{Results}

The screened bound failed on new starting states. In the \SNEval{}
evaluable conditions of the main study, it had \SShieldFailK{} violating
conditions (at most \SThresholdQ{} allowed), because the screen assigned
conditions with $\tilde z$ between \SScreenFailZLo{} and \SScreenFailZHi{}
to the Freedman branch, where the simplified bound was violated; it
\SEOne{} its violation criterion and \SEThree{} its assignment criterion. It
\SETwo{} the certification criterion, and in the horizon and 14B studies,
whose conditions contained fewer forgetting events, the violation criterion
was \HorizonBigCriterion{} (family-wise results in
Appendix~\ref{app:study20}).

The failures on new starting states had the same form as the earlier ones.
Across the \ProNConds{} conditions of the main, horizon and 14B studies and
the added Qwen3-8B batch,
every one of the \ProTerminalFacts{} facts on which the simplified bound was
violated had $R \ge a$, and the complete bound, the Azuma bound, the
per-fact rule and the margin threshold held in every condition
(Table~\ref{tab:main}). On the \TOneN{} conditions of the confirmatory test,
whose outcomes were unknown when we defined it, the complete bound had
\TOneK{} violating conditions, so it \TOneResult{} the violation criterion,
and \TTwoRge{} of the \TTwoFacts{} facts that violated the simplified bound
had $R \ge a$, so the location criterion was \TTwoResult{} on few facts.

This pattern explains the failure of the screened bound: a condition with a
large median $\tilde z$ can still contain a few facts with $R \ge a$, to
which the Freedman branch applies the simplified bound.

Over these studies, the per-fact rule certified \ProLicPerfact{} of
eligible facts, the complete bound \ProLicComplete, the Azuma bound
\ProLicAzuma{} and the margin threshold \ProLicMargin, but the per-fact rule
relies on a threshold chosen on earlier data. The updates learn task B: in
the main study, task-B accuracy rose by \STaskBGain{} percentage points on
average.

\subsection{Second confirmatory study}
\label{sec:rich}

All \RichN{} conditions were evaluable, but the new seeds gave fewer
forgetting events than earlier seeds at the same settings: \RichEvents{} in
total, and zero in \RichZeroEvent{} conditions (Table~\ref{tab:study25}). The simplified bound was violated in
\RichFailTerminal{} condition, on \RichLocFacts{} fact with $R \ge a$, so
the location criterion, which needs at least five violated facts, was
\RichLocResult. The complete bound had \RichValK{}
violating conditions (allowance \RichValQ), so it \RichValResult{} the
violation criterion. The Azuma bound, the per-time bound and the margin
threshold also held in every condition, and the per-fact rule was violated
on the same fact as the simplified bound. Restricted to facts with $R < a$,
the simplified bound held in every condition and certified
\RichLicTerminalStable{} of eligible facts, against \RichLicComplete{}
(complete) and \RichLicAzuma{} (Azuma).

\subsection{Pooled forgetting counts}
\label{sec:pooled}

\begin{table}[t]
\centering\small
\caption{Forgetting events pooled over all \AllNConds{} conditions (post
hoc). For each rule: the percentage of eligible facts certified, the number
of conditions with a violation, the pairs of a certified fact and a test
run, the forgetting events among them, and the events that the values of
the rule allow. The last two rows certify only facts with $R < a$.}
\label{tab:pooled}
\setlength{\tabcolsep}{4pt}
\begin{tabular}{lrrrrr}
\toprule
Rule & certified (\%) & viol.\ cond. & fact-runs & events & allowed \\
\midrule
Simplified Freedman ($R = 0$) & 81.0 & 14 & 118{,}928 & 344 & 467.5 \\
Screened bound & 67.2 & 6 & 98{,}480 & 202 & 314.8 \\
Per-time Freedman & 74.5 & 5 & 109{,}264 & 188 & 513.4 \\
Complete Freedman & 23.8 & 0 & 34{,}256 & 0 & 130.8 \\
Azuma & 17.3 & 0 & 24{,}688 & 0 & 178.8 \\
\midrule
Simplified Freedman on facts with $R < a$ & 31.0 & 0 & 44{,}912 & 0 & 29.8 \\
Every fact with $R < a$ (value 0.05) & 31.0 & 0 & 44{,}976 & 0 & 2248.8 \\
\bottomrule
\end{tabular}

\end{table}

A single condition detects only large excesses over a bound
(Appendix~\ref{app:power}), so Table~\ref{tab:pooled} pools the test runs of
all conditions. The facts that the complete bound certifies were forgotten
\PooledCompleteEvents{} times in \PooledCompleteRuns{} pairs of a fact and
a test run, against \PooledCompleteAllowed{} events that its values allow,
and the facts that the Azuma bound certifies \PooledAzumaEvents{} times.
The simplified bound produced fewer events than it allows in total, but its
violations lie in $R \ge a$. The last two rows separate the two parts of
the complete bound: the eligible facts with $R < a$ were forgotten
\PooledROnlyEvents{} times in \PooledROnlyRuns{} pairs, a result that
depends on $R$ and $a$ alone. The exclusion of
$R \ge a$, more than the variance estimate, therefore makes the complete
bound valid.

\section{Discussion and limitations}
\label{sec:discussion}

\paragraph{What the results support.}
Probes at the start of a LoRA fine-tuning run with SGD give per-fact bounds
whose formal guarantee covers a linear surrogate and which held for the
actual model in our conditions. The simplified Freedman bound certifies most
facts, but it failed only on facts whose coefficient change $R$ exceeds
their distance $a$ to the boundary, and these failures are breakdowns of
the first-order model. The complete bound, which excludes these facts, and
the more
conservative Azuma bound held in all \AllNConds{} conditions, including the
\ConfirmN{} confirmatory ones; the facts that the complete bound certifies
were forgotten \PooledCompleteEvents{} times in \PooledCompleteRuns{} pooled
pairs of a fact and a test run (Table~\ref{tab:pooled}). The
location of the failures rests mainly on post hoc data, because the
confirmatory studies contained only \ConfirmFacts{} violated facts. The
screened bound shows that a rule designed on earlier conditions can hold on
them and fail on new starting states, so only new starting states can test
it.

\paragraph{When the probes are worth their cost.}
A probability matters when the same update is run many times from one
starting state, for example when a shared model is adapted for many users.
A simulation then needs \SimRuns{} surviving runs to certify a fact at level
0.05 with 95\% confidence, whereas the probes cost a median of
\SProbeRuns{} test runs at $T = 40$ and \TProbeRuns{} at $T = 80$.

\paragraph{Limitations.}
All bounds hold for a linear surrogate, and the remainder that would
transfer them to the actual margin remains unquantified, so their coverage
on the actual model is an empirical finding; the variance estimate is also
low for most facts (Appendix~\ref{app:spread}). The complete bound excludes
every fact with $R \ge a$, although most of these facts survived every test
run. Certifying them needs a model that remains accurate along the test
runs, because the per-time bound failed on such facts. With 32 test runs
and a Bonferroni correction over 64 facts, a fact with bound 0.05 is flagged
only if it is forgotten in at least \PowerKFive{} runs, so a single
condition rules out only large excesses over a bound, and the pooled counts
carry the evidence about smaller ones. Our results cover SGD at
horizons of 40 and 80 steps with 64 synthetic or CounterFact facts per
condition; Adam and full fine-tuning remain open.

\section{Conclusion}
\label{sec:conclusion}

A first-order model of fine-tuning can be accurate and still fail to predict a
full run, yet it can bound the probability that a margin first falls below a
boundary near zero. The simplified Freedman bound failed only on facts with
$R \ge a$, a condition known before the run that every observed breakdown
of the first-order model met. The complete bound, which excludes these facts, and the Azuma
bound held in all \AllNConds{} conditions, whereas a screen that judged
whole conditions failed on new starting states. Bounds on forgetting should
therefore be computed for each starting state and each fact.

\ifdefined\arxivcopy\else
\section*{Acknowledgments}
Omitted for anonymous review.
\fi

\section*{Ethics statement}
This work uses only public model checkpoints and synthetic or public factual
data, and all experiments are computational. The main
ethical risk is over-trust: a bound on forgetting can be quoted outside the
assumptions under which it holds. A bound that certifies which knowledge
survives an update could also be used to protect content that should be
removed. We report the facts on which the simplified Freedman bound fails,
and we state the scope of every claim.

\section*{AI use statement}
We used large language models to help write and proofread the text. The
authors take full responsibility for the content.

\ifdefined\arxivcopy\else
\section*{Reproducibility statement}
The supplementary material contains the code, all preregistrations, the
frozen analysis scripts, the run directories with prediction and margin
files, and the script that generates the results, tables and figures of this
paper from those directories. Appendix~\ref{app:seeds} reports a rebuild of
two conditions from an empty copy of the repository; one of them gave a
different starting state, although the configuration, data, package
versions, model revision and GPU were the same.
\fi

\bibliographystyle{iclr2027_conference}
\bibliography{refs}
\appendix
\section{Proofs}
\label{app:theory}

\subsection{Theorem~\ref{thm:coupling} (coupling bound)}
Couple $P$ and $Q$ step by step. At each step where both chains are in
$\mathcal N$ and are still coupled, the kernels differ by at most
$\varepsilon_t$ in total variation, so a maximal coupling keeps them equal
except on an event of probability $\varepsilon_t$. Decoupling by time $T$
therefore has probability at most
$P(\tau \le T) + \sum_{t \le T}\varepsilon_t$, and on the complement the
chains agree on every path event. For tightness, take $A$ to be the
decoupling event in a system where leaving $\mathcal N$ determines the
outcome. \qed

\subsection{Theorem~\ref{thm:nonid} (non-identifiability)}
Fix a smooth bump function $\phi$ that vanishes on $\mathcal N$ together
with all its derivatives on the boundary of $\mathcal N$. Two drift fields
that agree on $\mathcal N$ and differ by $\lambda\phi$ outside it give
identical transition kernels for any path that stays in $\mathcal N$,
identical derivatives of all orders at every point of $\mathcal N$, and the
same noise law. Outside $\mathcal N$ the field is unconstrained, so
$\lambda$ can route the mass that has left $\mathcal N$ either across the
event boundary at step 2 or away from it, which moves the event probability by
exactly $p$. A predictor that sees only $\mathcal N$ has the same
distribution under both systems and errs by at least $p/2$ on one of them.
\qed

\subsection{Theorem~\ref{thm:tu} (first-passage bounds on the grid)}
Write $\widehat Y_t - m_t = S_t + D_t$ with
$S_t = \sum_{s\le t}\eta\, q_{T,s}^\top \xi_s$ and
$D_t = \sum_{s\le t}\eta\,(q_{t,s}-q_{T,s})^\top \xi_s$.

\emph{(i) $(S_u)_u$ is a martingale for
$\mathcal{F}_u = \sigma(\xi_1,\dots,\xi_u)$.} Each $q_{T,s}$ is
$\mathcal{F}_0$-measurable, because the response coefficients are
estimated by finite differences at the starting state and are fixed before
the test run starts, and $\mathbb{E}[\xi_s \mid \mathcal{F}_{s-1}] = 0$
because $\xi_s$ is the applied clipped update centered at its conditional
mean. Since $\|\xi_s\| \le 2G_0$, every increment satisfies
$|\eta\, q_{T,s}^\top \xi_s| \le 2\eta G_0 r = c$.

\emph{(ii) $|D_t| \le R$ almost surely.} By the Cauchy--Schwarz inequality,
$|D_t| \le \sum_{s\le t}\eta\lVert q_{t,s}-q_{T,s}\rVert\lVert\xi_s\rVert
\le 2\eta G_0 \max_{t\in G}\sum_{s< t} w_s \lVert q_{t,s}-q_{T,s}\rVert = R$,
with the interval widths $w_s$ of the probe grid.

\emph{(iii) Event inclusion.} If $\widehat Y_t < B_0$ for some $t \in G$, then
$S_t + D_t < B_0 - m_t \le -a$, so $-S_t > a - R = a'$. The first-passage
event is contained in $\{\exists u \le T : -S_u \ge a'\}$.

\emph{(iv) Maximal inequalities.} The maximal Azuma inequality for a
martingale with increments bounded by $c$ gives
$\Pr(\exists u \le T: -S_u \ge a') \le \exp\{-a'^2/(2Tc^2)\}$. If in addition
$\langle -S\rangle_T \le V$, the maximal Freedman inequality
\citep{freedman1975tail} gives the bound $\exp\{-a'^2/(2(V + c\,a'/3))\}$.
The maximal inequalities cover all grid times at once. For the surrogate with final-time
coefficients, $m_t + S_t$, we have $D_t = 0$, so the bounds hold with
$R = 0$. \qed

\subsection{Estimated variance (Section~\ref{sec:theory})}
Condition on $\mathcal F_0$, which fixes $\widehat V$ and the other bound
parameters, and partition on
$\Omega = \{\langle -S\rangle_T \le \widehat{V}\}$. The maximal Freedman
inequality bounds the probability of a crossing on $\Omega$ by
$\exp\{-a'^2/(2(\widehat{V}+c\,a'/3))\}$, and
$\Pr(\Omega^c\mid\mathcal F_0) \le \delta$ by assumption. An estimate
$\widehat V$ that depends on the outcomes of the test runs requires a
separate argument. \qed

\subsection{Selection before the run}
Let $\mathcal F_0$ be the information generated by the probes; the screen is
a deterministic function of $\mathcal F_0$.

\begin{proposition}[Selection before the run]
\label{prop:select}
Let $B \in \{\mathrm F, \mathrm A\}$ be the selected branch and let
$\mathcal L$ be any $\mathcal F_0$-measurable set of certified facts.
Suppose that on the event $\{B = b\}$, each branch bound satisfies
$\Pr(E_i \mid \mathcal F_0) \le U_i^b$, where $E_i$ is the forgetting event
of fact $i$. Then
$\Pr\bigl(\bigcup_{i \in \mathcal L} E_i\bigr) \le
\mathbb E\bigl[\sum_{i\in\mathcal L} U_i^{B}\bigr]$.
\end{proposition}

The proposition leaves the validity of each branch open; it shows that
choosing the branch before the run keeps the guarantee of the selected
branch.

\emph{Proof.} Let $A_b = \{B = b\}$. Because $B$ and $\mathcal L$ are
$\mathcal F_0$-measurable,
$\Pr\bigl((\bigcup_{i\in\mathcal L} E_i) \cap A_b\bigr) =
\mathbb E\bigl[\mathbf 1_{A_b}\Pr(\bigcup_{i\in\mathcal L} E_i \mid \mathcal F_0)\bigr]
\le \mathbb E\bigl[\mathbf 1_{A_b}\sum_{i\in\mathcal L} U_i^b\bigr]$
by the union bound inside the conditional probability. Summing over $b$
gives the stated bound with $U_i^B$. The argument holds for any dependence
between facts and between branches. \qed

\paragraph{Numerical checks.}
A Monte Carlo experiment checks the maximal inequality on clipped
heteroscedastic and two-point martingales (largest ratio of empirical
frequency to bound: 0.354). An independent reimplementation recomputes every
stored bound (largest discrepancy $1.7\times10^{-16}$).

\section{Scope of the first-order response model}
\label{app:local}
\begin{table}[h]
\centering\small
\caption{First-order response model at task-B learning rate \CanonicalLr. A
validation run passes a preregistered threshold if the correlation $\rho$
between predicted and measured margin changes under held-out perturbations
is at least 0.90 and the slope lies in $[0.8, 1.2]$. Every run passes; $n$
is the number of runs, and ranges are over runs.}
\label{tab:local}
\begin{tabular}{llrcc}
\toprule
Checkpoint & Family & $n$ & $\rho$ & slope \\
\midrule
Qwen3-0.6B & Qwen & 4 & 0.985--0.994 & 0.901--1.047 \\
Llama-3.2-1B & Llama & 3 & 0.974--0.998 & 0.974--1.028 \\
Qwen3-1.7B & Qwen & 4 & 0.982--0.996 & 0.986--1.075 \\
Qwen3-8B & Qwen & 1 & 0.994 & 0.987 \\
Gemma-3-12B & Gemma & 2 & 0.992--0.997 & 1.003--1.010 \\
Qwen3-14B & Qwen & 1 & 0.998 & 1.014 \\
\bottomrule
\end{tabular}

\end{table}

A random low-rank basis gives $\rho = 0.057$ because the margin changes that
it induces are smaller than the numerical noise of the forward pass. A
precision study that reruns the same 4-bit weights in 32-bit arithmetic
attributes this noise to the arithmetic rather than to weight quantization.
Under SGD, probe responses grow by a factor of 1.92 to 2.06 when the
perturbation size doubles, over the full horizon; under Adam this linear
scaling breaks down after about ten steps.

On Qwen3-0.6B, we validated the response model up to a parameter
displacement of 0.080 along the basis, with relative error 0.37 at that
radius. A single forgetting event needs a displacement of about 1.4 (median
margin 0.71 divided by median response norm 0.50), which is about
$18\times$ the validated radius. The task-B update noise of a 40-step run
alone accumulates a displacement of about 0.77, about ten times the
validated radius, so every run leaves the validated region, independently
of forgetting.

\section{Predictions of forgetting built on the response model}
\label{app:gap}

\begin{table}[h]
\centering\small
\caption{Five predictions built on the response model, each preregistered
with a success criterion and tested on data disjoint from the data that
suggested it. The outcome column gives the result that missed the
criterion: an error above the threshold, an effect outside the predicted
interval, a correlation that fell on replication, a rejected sufficiency
hypothesis, and a prediction of the size of the divergence that missed its
direction.}
\label{tab:gap}
\begin{tabular}{p{0.34\textwidth}p{0.2\textwidth}p{0.36\textwidth}}
\toprule
Prediction & Independent test & Outcome \\
\midrule
Per-fact event probabilities from the response model and a Gaussian noise
model & fresh split of facts & mean absolute error 0.060 against a
threshold of 0.05; all errors are false alarms \\
Rotating injected noise of fixed spectrum changes the forgetting rate by
0.145 between the extreme orientations & 32 paired runs, spectra
matched to $10^{-14}$ & difference 0.032, 95\% interval $[-0.043, 0.105]$,
excludes 0.145 \\
A drift statistic of the response coefficients ranks facts by prediction
error & fresh set of 128 facts & mean rank correlation falls from 0.639 to
0.075 \\
Two histories with the same margin vector have the same future distribution
& branching from matched states & energy-distance permutation test,
$p = 0.00035$ \\
A rank-2 projection of the protected gradients predicts how futures
diverge & runs branched from matched states & predicts how much futures diverge (Spearman
0.658, against 0.050 for a random projection) but misses the direction \\
\bottomrule
\end{tabular}
\end{table}

\section{History of the screening threshold}
\label{app:screen}
The screen comes from a rule that we preregistered before Studies 13 and
14: predict a violation of the simplified Freedman bound if the median of
$a_i/\sqrt{V_i/\kappa}$ over the certified facts is below 30. A first
version with an additional criterion failed its prospective test in Study 13
(four of five evaluable conditions correct). The simpler rule was correct in
six of six conditions of Study 14 when its predictions are recomputed from
the recorded statistic; the runner had stored the predictions of the first
version, a protocol deviation recorded in that preregistration, and those
stored labels are correct in five of six. The rule then failed a replication
across seeds that used an increment bound too small by a factor of two
(Appendix~\ref{app:record}). With the increment bound of
Equation~\ref{eq:c}, its only remaining error in that replication is
conservative: the bound held in one starting state that it flagged. The
screened bound differs from this rule in two ways. It uses the median of
$a_i/\sqrt{V_i}$, which is smaller by a factor of $\sqrt{\kappa} = \sqrt 2$,
so its threshold of 30 corresponds to about 42.4 on the earlier statistic.
It also uses the screen to select a bound rather than to predict a failure.
We fixed this definition in the code on 2026-09-15, before the prospective
studies.

Table~\ref{tab:sweep} shows, post hoc on the earlier conditions, how the
number of violating conditions and the certified fraction change with the
threshold and with uniform variance inflation. A threshold of 30 is the
smallest value in this grid at which every earlier condition passes.
Although the threshold was fixed before this table was computed, the
earlier data alone leave it indistinguishable from a fitted value, which is
why the prospective studies test it. Figure~\ref{fig:screen} shows the
result: the earlier failures lie below the threshold, and prospective
failures occurred above it. Uniform inflation of the variance estimate
leaves violating conditions even at $\kappa = \KappaMax$
(Table~\ref{tab:sweep}, right), which agrees with failures that are
confined to facts with $R \ge a$, on which the spread of the margin exceeds
the prediction by a larger factor than for other facts
(Appendix~\ref{app:spread}).

\begin{table}[h]
\centering\small
\caption{Post-hoc sweeps on the earlier conditions. Left: the screening
threshold. Right: the simplified Freedman bound with its variance estimate
inflated by $\kappa$ (the default is 2).}
\label{tab:sweep}
\begin{tabular}{lrr}
\toprule
Threshold on $\tilde z$ & violating cond. & certified \\
\midrule
0 (always simplified) & 8 & 79.6\% \\
10 & 8 & 79.6\% \\
15 & 6 & 77.1\% \\
20 & 2 & 67.2\% \\
25 & 1 & 66.2\% \\
30 & 0 & 62.3\% \\
35 & 0 & 59.6\% \\
40 & 0 & 54.1\% \\
50 & 0 & 45.0\% \\
60 & 0 & 34.8\% \\
80 & 0 & 29.7\% \\
always Azuma & 0 & 21.2\% \\
\bottomrule
\end{tabular}
\hspace{2em}\begin{tabular}{rrr}
\toprule
Inflation $\kappa$ & violating cond. & certified \\
\midrule
2 & 8 & 79.6\% \\
4 & 7 & 78.3\% \\
8 & 5 & 75.9\% \\
16 & 5 & 72.6\% \\
32 & 5 & 68.7\% \\
64 & 2 & 63.8\% \\
128 & 2 & 56.2\% \\
\bottomrule
\end{tabular}

\end{table}

\begin{figure}[h]
\centering
\includegraphics[width=0.76\textwidth]{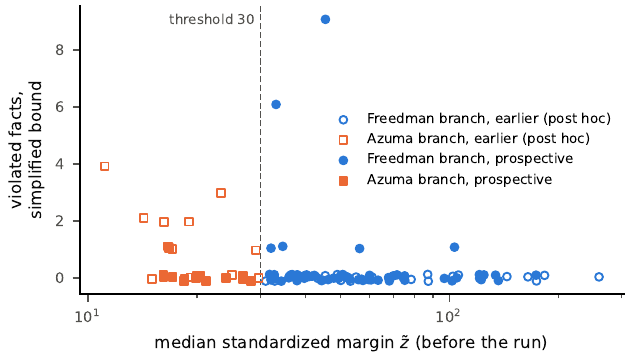}
\caption{Median standardized margin $\tilde z$ of the facts that the
simplified Freedman bound certifies, computed before the run, against the
number of facts on which that bound is significantly violated (with small
vertical jitter). Hollow: earlier conditions (post hoc); filled: all
\AllproNConds{} prospective conditions. The screen selects the Freedman
branch to the right of the dashed line.}
\label{fig:screen}
\end{figure}

\section{Calibration of the bounds}
\label{app:calibration}
Table~\ref{tab:calibration} groups the eligible facts of each set by the
value of their bound and compares the mean bound in each group with the mean
frequency of a forgetting event per test run. For a valid bound, the
frequency is at most the mean bound in every group, up to sampling error; a
group in which the frequency exceeds the mean bound shows where the bound
is too small.

\begin{table}[h]
\centering\scriptsize
\caption{Mean bound and frequency of forgetting events, grouped by the value
of the bound, for all \HistN{} earlier conditions and for the main and
horizon studies (the added 8B and 14B batches and the complete bound are
covered by Table~\ref{tab:main}).}
\label{tab:calibration}
\setlength{\tabcolsep}{3pt}
\begin{tabular}{llrrrrrr}
\toprule
& & \multicolumn{3}{c}{simplified Freedman bound} & \multicolumn{3}{c}{screened bound} \\
\cmidrule(lr){3-5}\cmidrule(lr){6-8}
Set & bound range & facts & mean bound & event freq. & facts & mean bound & event freq. \\
\midrule
Earlier & $[0, 0.001)$ & 1140 & 0.0001 & 0.0014 & 1009 & 0.0001 & 0.0000 \\
 & $[0.001, 0.01)$ & 270 & 0.0041 & 0.0038 & 151 & 0.0039 & 0.0000 \\
 & $[0.01, 0.05)$ & 196 & 0.0253 & 0.0110 & 96 & 0.0240 & 0.0020 \\
 & $[0.05, 0.2)$ & 179 & 0.1057 & 0.0239 & 83 & 0.1085 & 0.0041 \\
 & $[0.2, 0.5)$ & 91 & 0.3324 & 0.0797 & 67 & 0.3510 & 0.0420 \\
 & $[0.5, 1]$ & 141 & 0.8426 & 0.3510 & 611 & 0.8803 & 0.1022 \\
Main & $[0, 0.001)$ & 690 & 0.0001 & 0.0032 & 608 & 0.0001 & 0.0035 \\
 & $[0.001, 0.01)$ & 164 & 0.0042 & 0.0078 & 128 & 0.0042 & 0.0100 \\
 & $[0.01, 0.05)$ & 129 & 0.0238 & 0.0203 & 89 & 0.0250 & 0.0284 \\
 & $[0.05, 0.2)$ & 108 & 0.1095 & 0.0567 & 76 & 0.1049 & 0.0678 \\
 & $[0.2, 0.5)$ & 55 & 0.3321 & 0.1545 & 46 & 0.3393 & 0.1039 \\
 & $[0.5, 1]$ & 93 & 0.8422 & 0.3659 & 292 & 0.8696 & 0.1331 \\
Horizon & $[0, 0.001)$ & 190 & 0.0001 & 0.0000 & 148 & 0.0001 & 0.0000 \\
 & $[0.001, 0.01)$ & 40 & 0.0045 & 0.0000 & 38 & 0.0045 & 0.0000 \\
 & $[0.01, 0.05)$ & 13 & 0.0226 & 0.0000 & 14 & 0.0238 & 0.0000 \\
 & $[0.05, 0.2)$ & 8 & 0.1149 & 0.0000 & 4 & 0.1316 & 0.0000 \\
 & $[0.2, 0.5)$ & 4 & 0.3767 & 0.0000 & 15 & 0.3559 & 0.0000 \\
 & $[0.5, 1]$ & 3 & 0.8779 & 0.5729 & 39 & 0.8565 & 0.0441 \\
\bottomrule
\end{tabular}

\end{table}

\section{Earlier conditions}
\label{app:history}
Table~\ref{tab:history} lists all \HistN{} earlier conditions, with all
bounds recomputed with the increment bound of Equation~\ref{eq:c}. Entries
are violating facts over certified facts, and the branch is F (Freedman
branch) or A (Azuma branch). $\tilde z$ is the median standardized margin of
the facts that the simplified Freedman bound certifies. Studies 6 to 11
were single-seed studies of model size and learning rate, Studies 13 and 14
tested the earlier screening rule, Study 17 varied the seed at fixed
configurations, and Study 18 used CounterFact. The rebuild of
Appendix~\ref{app:seeds} that produced byte-identical files is counted once.
Four of these conditions lack the stored probe file from which $R$ is
computed; Section~\ref{sec:evidence} and Table~\ref{tab:main}
therefore use the other \EarNConds, while the threshold sweeps of
Appendix~\ref{app:screen} use all \HistN. In the evaluable conditions of
Studies 17 and 18, the reference margin $a_i$ ranked facts by forgetting
risk with a higher area under the ROC curve than the simplified Freedman
bound in five of six
conditions, so the reference margin is an informative ranking score,
although it provides a ranking rather than a probability.

{\scriptsize
\setlength{\tabcolsep}{3pt}
\begin{longtable}{llrrlrrrrr}
\caption{Earlier conditions (post hoc).}\label{tab:history}\\
\toprule
Study & Model & lr & seed & branch & $\tilde z$ & events & simplified & Azuma & screened \\
\midrule
\endfirsthead
\toprule
Study & Model & lr & seed & branch & $\tilde z$ & events & simplified & Azuma & screened \\
\midrule
\endhead
S18 & Llama-3.2-1B & 0.002 & 11 & F & 164.9 & 0 & 0/60 & 0/47 & 0/60 \\
S6 & Llama-3.2-1B & 0.002 & 11 & F & 183.3 & 0 & 0/64 & 0/64 & 0/64 \\
S13 & Llama-3.2-1B & 0.006 & 11 & F & 57.1 & 56 & 0/58 & 0/28 & 0/58 \\
S13 & Llama-3.2-1B & 0.006 & 11 & F & 52.0 & 16 & 0/56 & 0/23 & 0/56 \\
S18 & Llama-3.2-1B & 0.008 & 11 & F & 55.2 & 0 & 0/45 & 0/14 & 0/45 \\
S18 & Llama-3.2-1B & 0.008 & 43 & F & 87.1 & 12 & 0/46 & 0/16 & 0/46 \\
S14 & Llama-3.2-1B & 0.009 & 11 & F & 50.0 & 31 & 0/45 & 0/4 & 0/45 \\
S13 & Llama-3.2-1B & 0.012 & 11 & F & 39.3 & 8 & 0/31 & 0/2 & 0/31 \\
S9a & Llama-3.2-1B & 0.012 & 11 & F & 52.4 & 83 & 0/36 & 0/0 & 0/36 \\
S13 & Llama-3.2-1B & 0.012 & 23 & A & 18.5 & 74 & 0/32 & 0/0 & 0/0 \\
S13 & Llama-3.2-1B & 0.012 & 37 & A & 19.0 & 210 & 2/30 & 0/1 & 0/1 \\
S17 & Llama-3.2-1B & 0.012 & 41 & F & 68.2 & 0 & 0/30 & 0/12 & 0/30 \\
S17 & Llama-3.2-1B & 0.012 & 43 & A & 19.2 & 30 & 0/22 & 0/0 & 0/0 \\
S17 & Llama-3.2-1B & 0.012 & 47 & A & 15.1 & 111 & 0/24 & 0/0 & 0/0 \\
S17 & Llama-3.2-1B & 0.012 & 59 & F & 48.9 & 20 & 0/31 & 0/11 & 0/31 \\
S17 & Llama-3.2-1B & 0.012 & 61 & F & 144.2 & 0 & 0/31 & 0/12 & 0/31 \\
S17 & Llama-3.2-1B & 0.012 & 67 & A & 16.2 & 71 & 0/29 & 0/0 & 0/0 \\
S14 & Llama-3.2-1B & 0.015 & 11 & F & 40.8 & 180 & 0/34 & 0/0 & 0/34 \\
S13 & Llama-3.2-1B & 0.018 & 11 & F & 43.2 & 16 & 0/31 & 0/0 & 0/31 \\
S18 & Llama-3.2-1B & 0.02 & 11 & A & 29.1 & 83 & 1/30 & 0/0 & 0/0 \\
S13 & Llama-3.2-1B & 0.024 & 11 & F & 31.0 & 0 & 0/26 & 0/0 & 0/26 \\
S13 & Llama-3.2-1B & 0.03 & 11 & A & 25.0 & 0 & 0/25 & 0/0 & 0/0 \\
S18 & Qwen3-0.6B & 0.002 & 11 & F & 174.1 & 0 & 0/63 & 0/41 & 0/63 \\
S6 & Qwen3-0.6B & 0.002 & 11 & F & 259.4 & 5 & 0/62 & 0/58 & 0/62 \\
S13 & Qwen3-0.6B & 0.006 & 11 & F & 78.3 & 0 & 0/42 & 0/18 & 0/42 \\
S18 & Qwen3-0.6B & 0.008 & 11 & F & 53.2 & 1 & 0/48 & 0/13 & 0/48 \\
S18 & Qwen3-0.6B & 0.008 & 43 & F & 61.1 & 13 & 0/50 & 0/13 & 0/50 \\
S14 & Qwen3-0.6B & 0.012 & 11 & A & 17.1 & 116 & 1/26 & 0/1 & 0/1 \\
S13 & Qwen3-0.6B & 0.018 & 11 & A & 23.3 & 179 & 3/21 & 0/1 & 0/1 \\
S17 & Qwen3-0.6B & 0.018 & 41 & F & 39.7 & 233 & 0/29 & 0/3 & 0/29 \\
S17 & Qwen3-0.6B & 0.018 & 43 & F & 105.9 & 0 & 0/28 & 0/0 & 0/28 \\
S17 & Qwen3-0.6B & 0.018 & 47 & F & 102.2 & 0 & 0/32 & 0/6 & 0/32 \\
S17 & Qwen3-0.6B & 0.018 & 59 & F & 67.8 & 5 & 0/27 & 0/5 & 0/27 \\
S17 & Qwen3-0.6B & 0.018 & 61 & F & 36.5 & 0 & 0/30 & 0/0 & 0/30 \\
S17 & Qwen3-0.6B & 0.018 & 67 & F & 45.3 & 6 & 0/27 & 0/0 & 0/27 \\
S18 & Qwen3-0.6B & 0.02 & 11 & A & 29.6 & 53 & 0/24 & 0/0 & 0/0 \\
S14 & Qwen3-0.6B & 0.021 & 11 & A & 16.2 & 175 & 2/17 & 0/0 & 0/0 \\
S13 & Qwen3-0.6B & 0.024 & 11 & A & 11.1 & 98 & 4/20 & 0/0 & 0/0 \\
S13 & Qwen3-0.6B & 0.03 & 11 & A & 16.7 & 45 & 1/23 & 0/0 & 0/0 \\
S14 & Qwen3-1.7B & 0.006 & 11 & F & 87.4 & 78 & 0/48 & 0/17 & 0/48 \\
S14 & Qwen3-1.7B & 0.012 & 11 & F & 42.3 & 0 & 0/31 & 0/0 & 0/31 \\
S14 & Qwen3-1.7B & 0.018 & 11 & F & 32.8 & 28 & 0/27 & 0/0 & 0/27 \\
S14 & Qwen3-1.7B & 0.024 & 11 & F & 36.1 & 0 & 0/26 & 0/0 & 0/26 \\
S11b & Qwen3-14B & 0.008 & 11 & F & 44.5 & 34 & 0/51 & 0/9 & 0/51 \\
S10 & Qwen3-8B & 0.01 & 11 & A & 14.2 & 33 & 2/38 & 0/8 & 0/8 \\
\bottomrule
\end{longtable}

}

\section{Spread of the margins across test runs}
\label{app:spread}
For each eligible fact, we divide the standard deviation of the final margin
across the test runs by the standard deviation $\sqrt K$ that the
first-order model predicts. Over all \AllNConds{} conditions, the median of
this ratio was
\SpreadStable{} on facts with $R < a$, \SpreadUnstable{} on facts with
$R \ge a$, and \SpreadViolated{} on the facts that violated the simplified
bound. Of the \SpreadWideN{} facts with a ratio above 10, \SpreadWideRge{}
had $R \ge a$. The factor $\kappa = 2$ in $V$ allows for a ratio of up to
$\sqrt 2$, and \SpreadAboveKappa{} of the facts with $R < a$ exceeded it, so
the variance estimate is low for most facts. The complete bound held
nevertheless, because the facts that it certifies lie far from the boundary
compared with the realized spread: for \CompleteRealZThree{} of them,
$a - R$ exceeded three standard deviations of the realized final margin.

\section{Forgetting risk and the ratio $R/a$}
\label{app:ratio}
Figure~\ref{fig:ratio} groups the eligible facts of all \AllNConds{}
conditions with $a > 0$ by $R/a$. Of the \RatioBelowN{} facts with
$R < a$, \RatioBelowForgot{} were forgotten in any test run. Above $R = a$,
the share of facts forgotten in at least one test run rises steadily with
$R/a$, to \RatioTopShare{} for $R/a \ge 16$. The ratio therefore ranks facts
by forgetting risk before the run, and the bounds add a probability for the
facts that they certify. The remaining \RatioNearN{} eligible facts have
$a \le 0$: their reference margin stays positive but comes within $B_0$ of
zero, every bound leaves them uncertified, and \RatioNearForgot{} of them
were forgotten in at least one test run.

\begin{figure}[h]
\centering
\includegraphics[width=0.72\textwidth]{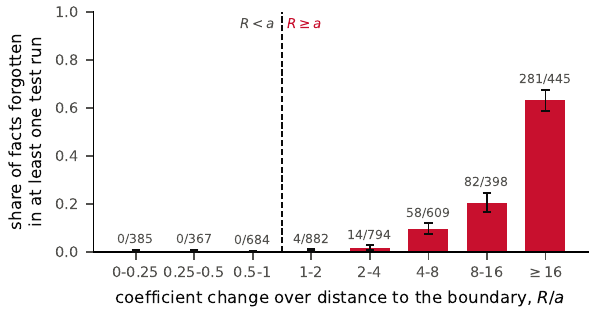}
\caption{Share of eligible facts forgotten in at least one test run, in
bins of $R/a$, over all \AllNConds{} conditions (facts with $a > 0$). Bars
show the share,
whiskers the Wilson 95\% interval, and labels the number of forgotten facts
over the number of facts in the bin.}
\label{fig:ratio}
\end{figure}

\section{Power of the test and pooled forgetting counts}
\label{app:power}
The per-fact test has limited power. With \PowerRuns{} test runs and a
Bonferroni correction over \PowerFamily{} facts, a fact with $U = 0.05$ is
flagged only if it is forgotten in at least \PowerKFive{} runs, and the test
reaches power 0.8 only when the true forgetting probability is at least
\PowerPEightyFive{} (Table~\ref{tab:power}). A condition that passes the
test therefore rules out only large excesses over the bound, and the
evidence about smaller excesses comes from pooling.

Table~\ref{tab:pooled} in Section~\ref{sec:pooled} gives the pooled counts.
The allowed number of events for a rule is the sum over its certified facts
of the number of runs times $U_i$. Events of different facts in one test run
are dependent, so we report these counts as descriptive totals. For the
margin threshold, the per-fact rule and the restricted margin threshold,
which the table omits, the certified facts were forgotten
\PooledMarginEvents{}, \PooledPerfactEvents{} and \PooledMarginREvents{}
times, against \PooledMarginAllowed, \PooledPerfactAllowed{} and
\PooledMarginRAllowed{} allowed events.

\begin{table}[h]
\centering\small
\caption{Power of the per-fact test with \PowerRuns{} test runs and a
Bonferroni correction over \PowerFamily{} facts: the smallest number of
forgetting events that is significant against $U$, the probability of
reaching it for true forgetting probabilities $p$, and the smallest $p$ at
which this probability reaches 0.8.}
\label{tab:power}
\begin{tabular}{rrrrrrr}
\toprule
$U$ & min.\ events & $p = 0.1$ & $p = 0.2$ & $p = 0.3$ & $p = 0.5$ & $p$ at power 0.8 \\
\midrule
0.05 & 8 & 0.01 & 0.30 & 0.79 & 1.00 & 0.304 \\
0.01 & 4 & 0.40 & 0.91 & 0.99 & 1.00 & 0.166 \\
0.001 & 2 & 0.84 & 0.99 & 1.00 & 1.00 & 0.091 \\
\bottomrule
\end{tabular}

\end{table}

\section{Variation across seeds and builds}
\label{app:seeds}
Outcomes vary strongly across seeds of one configuration. Among the earlier
conditions at fixed configurations, the number of forgetting events of
eligible facts ranged from \FixedEventsLo{} to \FixedEventsHi{} across seeds
for Llama-3.2-1B and from \SeedQwenEventsLo{} to \SeedQwenEventsHi{} for
Qwen3-0.6B, and the new seeds of the second confirmatory study gave far
fewer events than earlier seeds at the same settings. The seed sets both
the starting state and the minibatch orders of the test runs, so these
ranges mix the two effects. We also rebuilt two earlier conditions from an empty copy of the
repository. One gave byte-identical files. The other gave a different
starting state (mean margin of the protected facts 0.658 instead of 0.608)
and a different outcome, although the configuration, data, package
versions, model revision and GPU were the same. Bounds must therefore be
computed for each starting state.

\section{Prospective conditions}
\label{app:study20}
The family-wise statements of the screened bound were slightly above their
nominal level. In the main study, the largest set of facts per condition
whose screened bounds sum to at most 0.05 contained \SFamilyLic{} of
eligible facts, and some fact of this set was forgotten in \SFamilyExit{} of
\SFamilyTraj{} test runs (\SFamilyRate{} against 5\%, one-sided
$p = \SFamilyP$).

\begin{figure}[h]
\centering
\includegraphics[width=\textwidth]{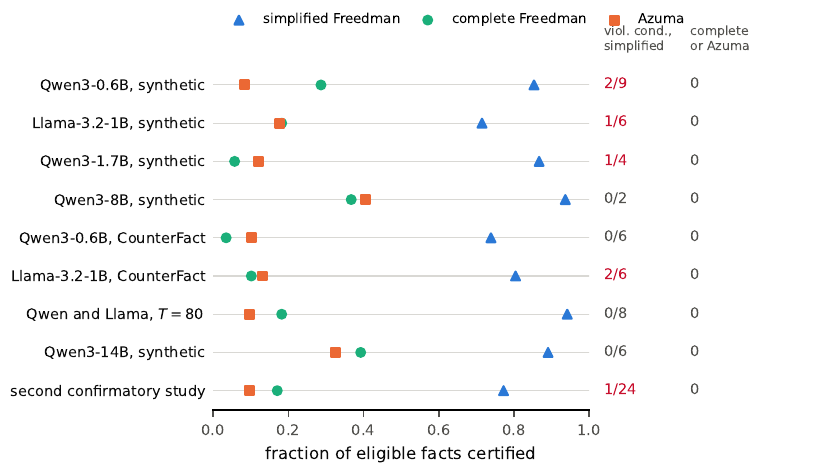}
\caption{Fraction of eligible facts certified by the simplified Freedman
bound, the complete Freedman bound and the Azuma bound in each condition set
of the main, horizon and 14B studies and in the second confirmatory study,
with the number of conditions with a violation (right).}
\label{fig:sets}
\end{figure}

{\scriptsize
\setlength{\tabcolsep}{3pt}
\begin{longtable}{lrrlrrrrr}
\caption{Main prospective study, every condition ($T = 40$).}\label{tab:study20cond}\\
\toprule
Set & lr & seed & branch & $\tilde z$ & events & simplified & Azuma & screened \\
\midrule
\endfirsthead
\toprule
Set & lr & seed & branch & $\tilde z$ & events & simplified & Azuma & screened \\
\midrule
\endhead
A & 0.012 & 71 & F & 173.7 & 0 & 0/32 & 0/13 & 0/32 \\
A & 0.012 & 73 & F & 56.4 & 50 & 1/33 & 0/4 & 1/33 \\
A & 0.012 & 79 & F & 45.4 & 242 & 9/31 & 0/1 & 9/31 \\
A & 0.018 & 71 & F & 122.5 & 0 & 0/32 & 0/2 & 0/32 \\
A & 0.018 & 73 & F & 37.8 & 0 & 0/32 & 0/3 & 0/32 \\
A & 0.018 & 79 & F & 69.9 & 267 & 0/17 & 0/0 & 0/17 \\
A & 0.021 & 71 & F & 101.7 & 0 & 0/32 & 0/0 & 0/32 \\
A & 0.021 & 73 & F & 38.0 & 0 & 0/30 & 0/2 & 0/30 \\
A & 0.021 & 79 & F & 63.4 & 44 & 0/17 & 0/0 & 0/17 \\
B & 0.012 & 71 & F & 68.2 & 132 & 0/30 & 0/21 & 0/30 \\
B & 0.012 & 73 & F & 44.3 & 5 & 0/31 & 0/7 & 0/31 \\
B & 0.012 & 79 & A & 24.1 & 49 & 0/17 & 0/0 & 0/0 \\
B & 0.018 & 71 & A & 26.9 & 88 & 0/28 & 0/7 & 0/7 \\
B & 0.018 & 73 & F & 38.9 & 0 & 0/32 & 0/3 & 0/32 \\
B & 0.018 & 79 & A & 16.6 & 45 & 1/15 & 0/0 & 0/0 \\
C & 0.012 & 71 & F & 46.2 & 0 & 0/29 & 0/3 & 0/29 \\
C & 0.012 & 73 & F & 33.1 & 160 & 6/49 & 0/16 & 6/49 \\
C & 0.021 & 71 & A & 20.5 & 0 & 0/29 & 0/0 & 0/0 \\
C & 0.021 & 73 & A & 16.2 & 177 & 0/30 & 0/0 & 0/0 \\
D & 0.01 & 71 & \multicolumn{6}{l}{excluded (task A)} \\
D & 0.01 & 73 & \multicolumn{6}{l}{excluded (task A)} \\
E & 0.008 & 71 & F & 52.5 & 2 & 0/46 & 0/8 & 0/46 \\
E & 0.008 & 73 & F & 49.9 & 42 & 0/50 & 0/12 & 0/50 \\
E & 0.008 & 79 & F & 59.8 & 25 & 0/47 & 0/10 & 0/47 \\
E & 0.02 & 71 & A & 26.9 & 67 & 0/22 & 0/0 & 0/0 \\
E & 0.02 & 73 & F & 32.9 & 69 & 0/26 & 0/0 & 0/26 \\
E & 0.02 & 79 & A & 21.2 & 74 & 0/24 & 0/0 & 0/0 \\
F & 0.008 & 71 & F & 49.3 & 12 & 0/47 & 0/12 & 0/47 \\
F & 0.008 & 73 & F & 53.0 & 15 & 0/47 & 0/11 & 0/47 \\
F & 0.008 & 79 & F & 60.7 & 12 & 0/47 & 0/13 & 0/47 \\
F & 0.02 & 71 & F & 32.1 & 64 & 0/26 & 0/0 & 0/26 \\
F & 0.02 & 73 & F & 32.1 & 61 & 1/27 & 0/0 & 1/27 \\
F & 0.02 & 79 & F & 34.6 & 51 & 1/28 & 0/0 & 1/28 \\
\bottomrule
\end{longtable}

}

\begin{table}[h]
\centering\scriptsize
\caption{Horizon study ($T = 80$). The branch is given for the paired
condition of the main study ($T = 40$) and for $T = 80$. Entries are
significant violations over certified facts at $T = 80$; the last column is
the number of facts that the screened bound certifies at $T = 40$.}
\label{tab:study21}
\setlength{\tabcolsep}{3pt}
\begin{tabular}{lrrlrrrrrr}
\toprule
Model & lr & seed & branch 40/80 & $\tilde z$ & events & simplified & Azuma & screened & screened at 40 \\
\midrule
Qwen3-0.6B & 0.012 & 71 & F/F & 136.6 & 0 & 0/32 & 0/6 & 0/32 & 32 \\
Qwen3-0.6B & 0.012 & 73 & F/F & 51.9 & 0 & 0/32 & 0/4 & 0/32 & 33 \\
Qwen3-0.6B & 0.018 & 71 & F/F & 124.6 & 0 & 0/32 & 0/0 & 0/32 & 32 \\
Qwen3-0.6B & 0.018 & 73 & F/F & 39.6 & 0 & 0/32 & 0/0 & 0/32 & 32 \\
Llama-3.2-1B & 0.012 & 71 & F/A & 26.9 & 30 & 0/27 & 0/6 & 0/6 & 30 \\
Llama-3.2-1B & 0.012 & 73 & F/F & 50.6 & 0 & 0/32 & 0/7 & 0/32 & 31 \\
Llama-3.2-1B & 0.018 & 71 & A/A & 17.1 & 25 & 0/24 & 0/2 & 0/2 & 7 \\
Llama-3.2-1B & 0.018 & 73 & F/F & 35.9 & 0 & 0/32 & 0/0 & 0/32 & 32 \\
\bottomrule
\end{tabular}

\end{table}

\begin{table}[h]
\centering\scriptsize
\caption{Qwen3-14B study ($T = 40$), two preregistered batches (Studies 22
and 23). Conditions below the task-A accuracy threshold are marked as
excluded. Entries are violating facts over certified facts.}
\label{tab:study22}
\setlength{\tabcolsep}{3pt}
\begin{tabular}{rrrlrrrrr}
\toprule
Study & lr & seed & branch & $\tilde z$ & events & simplified & Azuma & screened \\
\midrule
22 & 0.008 & 71 & \multicolumn{6}{l}{excluded (task A)} \\
22 & 0.012 & 71 & \multicolumn{6}{l}{excluded (task A)} \\
22 & 0.008 & 73 & F & 121.2 & 14 & 0/47 & 0/40 & 0/47 \\
22 & 0.012 & 73 & F & 32.6 & 30 & 0/44 & 0/16 & 0/44 \\
23 & 0.008 & 79 & F & 31.8 & 11 & 0/54 & 0/28 & 0/54 \\
23 & 0.012 & 79 & A & 18.4 & 65 & 0/37 & 0/16 & 0/16 \\
23 & 0.008 & 83 & F & 37.9 & 0 & 0/64 & 0/10 & 0/64 \\
23 & 0.012 & 83 & A & 20.0 & 24 & 0/56 & 0/0 & 0/0 \\
\bottomrule
\end{tabular}

\end{table}

\begin{table}[h]
\centering\scriptsize
\caption{Added Qwen3-8B batch ($T = 40$, Study 24). Both Qwen3-8B conditions
of the main study were below the task-A accuracy threshold
(Table~\ref{tab:study20cond}); this batch forms the Qwen3-8B row of
Table~\ref{tab:main}. Entries are violating facts over certified facts.}
\label{tab:study24}
\setlength{\tabcolsep}{3pt}
\begin{tabular}{rrlrrrrr}
\toprule
lr & seed & branch & $\tilde z$ & events & simplified & Azuma & screened \\
\midrule
0.01 & 79 & F & 47.8 & 0 & 0/46 & 0/20 & 0/46 \\
0.01 & 83 & F & 50.7 & 8 & 0/28 & 0/12 & 0/28 \\
\bottomrule
\end{tabular}

\end{table}

\begin{table}[h]
\centering\scriptsize
\caption{Second confirmatory study ($T = 40$, Study 25). Entries are
violating facts over certified facts for the simplified, complete and Azuma
bounds, the per-time bound, and the simplified bound restricted to facts
with $R < a$. The last column is the number of facts violated by the
simplified bound that have $R \ge a$.}
\label{tab:study25}
\setlength{\tabcolsep}{3pt}
\begin{tabular}{lrrrrrrrrr}
\toprule
model & lr & seed & events & simplified & complete & Azuma & per-time & simplified, $R < a$ & $R \ge a$ \\
\midrule
Qwen3-0.6B & 0.012 & 89 & 0 & 0/32 & 0/18 & 0/0 & 0/29 & 0/26 & 0 \\
Qwen3-0.6B & 0.012 & 97 & 0 & 0/32 & 0/5 & 0/2 & 0/32 & 0/8 & 0 \\
Qwen3-0.6B & 0.012 & 101 & 0 & 0/31 & 0/19 & 0/0 & 0/30 & 0/26 & 0 \\
Qwen3-0.6B & 0.012 & 103 & 20 & 0/35 & 0/18 & 0/4 & 0/31 & 0/22 & 0 \\
Qwen3-0.6B & 0.012 & 107 & 36 & 0/28 & 0/8 & 0/0 & 0/28 & 0/15 & 0 \\
Qwen3-0.6B & 0.012 & 109 & 65 & 0/26 & 0/6 & 0/7 & 0/23 & 0/12 & 0 \\
Qwen3-0.6B & 0.018 & 89 & 0 & 0/30 & 0/0 & 0/0 & 0/22 & 0/18 & 0 \\
Qwen3-0.6B & 0.018 & 97 & 0 & 0/22 & 0/0 & 0/0 & 0/21 & 0/0 & 0 \\
Qwen3-0.6B & 0.018 & 101 & 28 & 0/29 & 0/0 & 0/0 & 0/25 & 0/6 & 0 \\
Qwen3-0.6B & 0.018 & 103 & 100 & 0/30 & 0/10 & 0/0 & 0/29 & 0/13 & 0 \\
Qwen3-0.6B & 0.018 & 107 & 0 & 0/25 & 0/0 & 0/0 & 0/19 & 0/0 & 0 \\
Qwen3-0.6B & 0.018 & 109 & 0 & 0/21 & 0/4 & 0/3 & 0/17 & 0/5 & 0 \\
Qwen3-0.6B & 0.021 & 89 & 0 & 0/30 & 0/0 & 0/0 & 0/21 & 0/12 & 0 \\
Qwen3-0.6B & 0.021 & 97 & 0 & 0/20 & 0/0 & 0/0 & 0/17 & 0/0 & 0 \\
Qwen3-0.6B & 0.021 & 101 & 0 & 0/29 & 0/0 & 0/0 & 0/24 & 0/3 & 0 \\
Qwen3-0.6B & 0.021 & 103 & 71 & 0/28 & 0/2 & 0/0 & 0/28 & 0/6 & 0 \\
Qwen3-0.6B & 0.021 & 107 & 0 & 0/24 & 0/0 & 0/0 & 0/18 & 0/0 & 0 \\
Qwen3-0.6B & 0.021 & 109 & 1 & 0/18 & 0/2 & 0/1 & 0/16 & 0/5 & 0 \\
Qwen3-1.7B & 0.012 & 89 & 22 & 0/34 & 0/8 & 0/17 & 0/33 & 0/10 & 0 \\
Qwen3-1.7B & 0.012 & 97 & 30 & 0/40 & 0/3 & 0/4 & 0/40 & 0/8 & 0 \\
Qwen3-1.7B & 0.012 & 101 & 0 & 0/32 & 0/3 & 0/0 & 0/32 & 0/9 & 0 \\
Qwen3-1.7B & 0.012 & 103 & 94 & 0/34 & 0/7 & 0/24 & 0/32 & 0/7 & 0 \\
Qwen3-1.7B & 0.012 & 107 & 0 & 0/32 & 0/30 & 0/23 & 0/32 & 0/30 & 0 \\
Qwen3-1.7B & 0.012 & 109 & 1 & 1/27 & 0/9 & 0/1 & 0/27 & 0/10 & 1 \\
\bottomrule
\end{tabular}

\end{table}

\clearpage
\section{Preregistration record}
\label{app:record}
In the supplementary material, the main prospective study is Study 20, the
horizon study is Study 21, the two batches of the 14B study are Studies 22
and 23, the added Qwen3-8B batch is Study 24, and the second confirmatory
study is Study 25. Three addenda to Study 20
add the margin threshold and the per-fact rule (Addenda 1 and 2, committed
after the first condition had finished and its result had been seen) and the
confirmatory test of the complete bound (Addendum 3, committed after the
interim look, with the list of excluded runs).
The material also contains the preregistrations of all earlier studies, the
frozen analysis scripts, the verbatim results of every preregistered
criterion including the failed ones, and the script that generates every
number in this paper from the run directories.

During development we found and corrected three defects, and all numbers in
this paper use the corrected versions. First, the key of a probe cache
omitted the length of task A, so probes of one starting state could be used
for a run from another; we excluded the affected runs and kept them with a
written reason. Second, a rule that selected the newest run directory could
mix two concurrent experiments; we excluded those runs as well. Third,
earlier code used the increment bound $\eta G_0 r$ instead of $2\eta G_0 r$
(Equation~\ref{eq:c}); we recomputed every bound in this paper with the
correct value.

\end{document}